\documentclass{article}

\usepackage{microtype}
\usepackage{graphicx}
\usepackage{subcaption}
\usepackage[colorinlistoftodos]{todonotes}
\usepackage[colorlinks=true,citecolor=blue,linkcolor=blue]{hyperref}
\usepackage{amsmath}
\usepackage{amsthm}
\usepackage{amsfonts}
\usepackage{grffile}
\usepackage{mathtools}
\usepackage{graphicx}
\usepackage{paralist}
\usepackage{bm} 
\usepackage{booktabs} 
\usepackage{multirow}
\usepackage{makecell}
\usepackage[english]{babel}
\usepackage[T1]{fontenc}
\usepackage{bbm}
\usepackage{etoolbox}
\usepackage{enumitem}
\usepackage{comment}
\usepackage{xspace}
\usepackage{chngcntr}
\usepackage{adjustbox}
\newcommand{\fastlin}{\textbf{Fast-Lin}\xspace}
\newcommand{\fastlip}{\textbf{Fast-Lip}\xspace}
\newcommand{\reluplex}{\textbf{Reluplex}\xspace}
\newcommand{\lp}{\textbf{LP}\xspace}
\newcommand{\lpfull}{\textbf{LP-Full}\xspace}
\newcommand{\opnorm}{\textbf{Op-norm}\xspace}
\newcommand{\clever}{\textbf{CLEVER}\xspace}

\makeatletter
\newcommand\remembertext[2]{
  \immediate\write\@auxout{\unexpanded{\global\long\@namedef{mytext@#1}{#2}}}%
  #2%
}

\newcommand\recalltext[1]{%
  \ifcsname mytext@#1\endcsname
    \@nameuse{mytext@#1}%
  \else
    ``?? Undefined texts ??''
  \fi
}
\makeatother

\def \useicmlformat {}

\usepackage{hyperref}
\usepackage{setspace} 

\ifdef{\useicmlformat}{
\usepackage[accepted]{icml2018}
}{
\let\small\undefined
\let\footnotesize\undefined
\let\scriptsize\undefined
\let\tiny\undefined
\let\large\undefined
\let\Large\undefined
\let\LARGE\undefined
\let\huge\undefined
\let\Huge\undefined
\makeatletter
\input{size11.clo}
\makeatother
\usepackage[margin=1in]{geometry}
\usepackage{algorithm}
\usepackage[noend]{algpseudocode}
}

\newtheorem{theorem}{Theorem}[section]
\newtheorem{lemma}[theorem]{Lemma}
\newtheorem{definition}[theorem]{Definition}

\newtheorem{proposition}[theorem]{Proposition}
\newtheorem{corollary}[theorem]{Corollary}

\newtheorem{claim}[theorem]{Claim}

\newtheorem{hypothesis}[theorem]{Hypothesis}
\makeatletter
\newcommand*{\RN}[1]{\expandafter\@slowromancap\romannumeral #1@}
\makeatother

\newcommand{\wt}{\widetilde}

\newcommand{\eps}{\epsilon}

\newcommand{\R}{\mathbb{R}}

\renewcommand{\varepsilon}{\epsilon}
\renewcommand{\tilde}{\wt}

\renewcommand{\R}{\mathbb{R}}

\begin{document}

\ifdef{\useicmlformat}
{
\twocolumn[
\icmltitle{Towards Fast Computation of Certified Robustness for ReLU Networks}



\icmlsetsymbol{equal}{*}


\begin{icmlauthorlist}
\icmlauthor{Tsui-Wei Weng}{equal,mit}
\icmlauthor{Huan Zhang}{equal,ucd}
\icmlauthor{Hongge Chen}{mit}
\ifdef{\arxivversion}{\icmlauthor{Zhao Song\textsuperscript{$\dagger$}}{harvard,ut}\\
}
{\icmlauthor{Zhao Song}{harvard,ut}}
\icmlauthor{Cho-Jui Hsieh}{ucd}
\icmlauthor{Duane Boning}{mit}
\icmlauthor{Inderjit S. Dhillon}{ut}
\icmlauthor{Luca Daniel}{mit}
\ifdef{\arxivversion}{
\\ \textsuperscript{1}MIT \quad
\textsuperscript{2} UC Davis \quad
\textsuperscript{3}Harvard University\quad
\textsuperscript{4}UT Austin \quad \\
\texttt{\small twweng@mit.edu, ecezhang@ucdavis.edu, chenhg@mit.edu, zhaos@g.harvard.edu,}\\
\texttt{\small chohsieh@ucdavis.edu, dluca@mit.edu, boning@mtl.mit.edu, inderjit@cs.utexas.edu}\\
}{}
\end{icmlauthorlist}

\icmlaffiliation{mit}{Massachusetts Institute of Technology, Cambridge, MA}
\icmlaffiliation{ucd}{UC Davis, Davis, CA}
\icmlaffiliation{harvard}{Harvard University, Cambridge, MA}
\icmlaffiliation{ut}{UT Austin, Austin, TX. 
Source code is available at \url{https://github.com/huanzhang12/CertifiedReLURobustness}}

\icmlcorrespondingauthor{Tsui-Wei Weng}{twweng@mit.edu}
\icmlcorrespondingauthor{Huan Zhang}{\mbox{huan@huan-zhang.com}}

\icmlkeywords{Deep Neural Networks, ReLU, Robustness}

\vskip 0.3in
]




\ifdef{\arxivversion}{
\renewcommand*{\thefootnote}{\fnsymbol{footnote}}
\stepcounter{footnote}
\footnotetext{Tsui-Wei Weng and Huan Zhang contributed equally.}
\stepcounter{footnote}
\footnotetext{Part of the work done while hosted by Jelani Nelson.}
\renewcommand*{\thefootnote}{\arabic{footnote}}
\setcounter{footnote}{0}
}{
\printAffiliationsAndNotice{\icmlEqualContribution} 
}

\begin{abstract}
Verifying the robustness property of a general Rectified Linear Unit (ReLU) network is an NP-complete problem. Although finding the \textit{exact} minimum adversarial distortion is hard, giving a \textit{certified lower bound} of the minimum distortion is possible. Current available methods of computing such a bound are either time-consuming or deliver low quality bounds that are too loose to be useful. In this paper, we exploit the special structure of ReLU networks and provide two computationally efficient algorithms (\textbf{Fast-Lin},\textbf{Fast-Lip}) that are able to certify non-trivial lower bounds of minimum adversarial distortions. Experiments show that (1) our methods deliver bounds close to (the gap is 2-3X) exact minimum distortions found by Reluplex in small networks while our algorithms are more than 10,000 times faster; (2) our methods deliver similar quality of bounds (the gap is within 35\% and usually around 10\%; sometimes our bounds are even better) for larger networks compared to the methods based on solving linear programming problems but our algorithms are 33-14,000 times faster; (3) our method is capable of solving large MNIST and CIFAR networks up to 7 layers with more than 10,000 neurons within tens of seconds on a single CPU core. In addition, we show that there is no polynomial time algorithm that can approximately find the minimum $\ell_1$ adversarial distortion of a ReLU network with a $0.99\ln n$ approximation ratio unless $\mathsf{NP}$=$\mathsf{P}$, where $n$ is the number of neurons in the network. 

\end{abstract}
}{
\title{Towards Fast Computation of Certified Robustness\\ for ReLU Networks}
\date{}
\author{
\centering
Tsui-Wei Weng\thanks{Tsui-Wei Weng and Huan Zhang contributed equally.} \thanks{ MIT, \texttt{twweng@mit.edu} }
\and
Huan Zhang\footnotemark[1] \thanks{ UC Davis, \texttt{ecezhang@ucdavis.edu}}
\and
Hongge Chen\thanks{ MIT, \texttt{chenhg@mit.edu} }
\and
Zhao Song\thanks{Harvard University \& UT-Austin, \texttt{zhaos@g.harvard.edu}, part of the work done while hosted by Jelani Nelson.}
\and
Cho-Jui Hsieh\thanks{UC Davis, \texttt{chohsieh@ucdavis.edu}}
\and
Duane Boning\thanks{ MIT, \texttt{boning@mtl.mit.edu} }
\and
Inderjit S. Dhillon\thanks{ UT-Austin, \texttt{inderjit@cs.utexas.edu} }
\and
Luca Daniel\thanks{ MIT, \texttt{dluca@mit.edu} }
}

\begin{titlepage}
 \maketitle
  \begin{abstract}
Verifying the robustness property of a general Rectified Linear Unit (ReLU) network is an NP-complete problem. Although finding the \textit{exact} minimum adversarial distortion is hard, giving a \textit{certified lower bound} of the minimum distortion is possible. Current available methods of computing such a bound are either time-consuming or deliver low quality bounds that are too loose to be useful. In this paper, we exploit the special structure of ReLU networks and provide two computationally efficient algorithms (\textbf{Fast-Lin},\textbf{Fast-Lip}) that are able to certify non-trivial lower bounds of minimum adversarial distortions. Experiments show that (1) our methods deliver bounds close to (the gap is 2-3X) exact minimum distortions found by Reluplex in small networks while our algorithms are more than 10,000 times faster; (2) our methods deliver similar quality of bounds (the gap is within 35\% and usually around 10\%; sometimes our bounds are even better) for larger networks compared to the methods based on solving linear programming problems but our algorithms are 33-14,000 times faster; (3) our method is capable of solving large MNIST and CIFAR networks up to 7 layers with more than 10,000 neurons within tens of seconds on a single CPU core. In addition, we show that there is no polynomial time algorithm that can approximately find the minimum $\ell_1$ adversarial distortion of a ReLU network with a $0.99\ln n$ approximation ratio unless $\mathsf{NP}$=$\mathsf{P}$, where $n$ is the number of neurons in the network. 


  \end{abstract}
 \thispagestyle{empty}
 \end{titlepage}
}

\ifdef{\arxivversion}{
{\hypersetup{linkcolor=black}
\tableofcontents
}
\newpage
}


\section{Introduction}
Since the discovery of adversarial examples in deep neural network (DNN) image classifiers \cite{szegedy2013intriguing}, researchers have successfully found adversarial examples in many machine learning tasks applied to different areas, including object detection \cite{xie2017adversarial}, image captioning \cite{chen2017show},
speech recognition \cite{cisse2017houdini}, malware detection \cite{wang2017adversary} and reading comprehension \cite{jia2017adversarial}.  
Moreover, black-box attacks have also been shown to be possible, where an attacker can find adversarial examples without knowing the architecture and parameters of the DNN \cite{CPY17_zoo,papernot2017practical,liu2016delving}.

The existence of adversarial examples poses a huge threat to the application of DNNs in mission-critical tasks including security cameras, self-driving cars and aircraft control systems. Many researchers have thus proposed defensive or detection methods in order to increase the robustness of DNNs. Notable examples are defensive distillation \cite{papernot2016distillation},
adversarial retraining/training \cite{kurakin2016adversarial_ICLR,madry2017towards} and model ensembles \cite{tramer2017ensemble,liu2017towards}.  Despite many published contributions that aim at increasing the robustness of DNNs, theoretical results are rarely given and there is no guarantee that the proposed defensive methods can reliably improve the robustness.  Indeed, many of these defensive mechanism have been shown to be ineffective when more advanced attacks are used \cite{carlini2017towards,carlini2017adversarial,carlini2017magnet,he2017adversarial}. 

The robustness of a DNN can be verified by examining a neighborhood (e.g. $\ell_2$ or $\ell_\infty$ ball) near a data point $\bm{x_0}$. The idea is to find the largest ball with radius $r_0$ that guarantees no points inside the neighborhood can ever change classifier decision. Typically, $r_0$ can be found as follows: given $R$, a global optimization algorithm can be used to find an adversarial example within this ball, and thus bisection on $R$ can produce $r_0$. Reluplex \cite{katz2017reluplex} is one example using such a technique but it is computationally infeasible even on a small MNIST classifier. In general, verifying the robustness property of a ReLU network is NP-complete \cite{katz2017reluplex, sinha2017certifiable}.

On the other hand, a lower bound $\beta_L$ of radius $r_0$ can be given, which guarantees that no examples within a ball of radius $\beta_L$ can ever change the network classification outcome. \cite{hein2017formal} is a pioneering work on giving such a lower bound for neural networks that are continuously differentiable, although only a 2-layer MLP network with differentiable activations is investigated. 
\cite{weng2017evaluating} has extended theoretical result to ReLU activation functions and proposed a robustness score, CLEVER, based on extreme value theory. Their approach is feasible for large state-of-the-art DNNs but CLEVER is an estimate of $\beta_L$ without certificates. Ideally, we would like to obtain a {\it certified} (which guarantees that $\beta_L \leq r_0$) and {\it non-trivial} (a trivial $\beta_L$ is 0) lower bound $\beta_L$ that is reasonably close to $r_0$ within {\it reasonable amount of computational time}.

In this paper, we develop two fast algorithms for obtaining a \textit{tight} and \textit{certified} lower bound $\beta_L$ on ReLU networks. In addition, we also provide a complementary theoretical result to \cite{katz2017reluplex,sinha2017certifiable} by further showing  there does not even exist a polynomial time algorithm that can approximately find the minimum adversarial distortion with a $0.99 \ln n$ approximation ratio. Our contributions are:

\begin{itemize}[nosep,wide,labelindent=0pt,labelwidth=*,align=left]
\item We fully exploit the ReLU networks to give two computationally efficient methods of computing \textit{tighter} and \textit{guaranteed} robustness lower bounds via (1) \textbf{lin}ear approximation on the ReLU units (see Sec~\ref{sec3:convexbnd}, Algorithm~\ref{alg:fast-lin}, \textbf{Fast-Lin}) 
and (2) bounding network local \textbf{Lip}schitz constant (see Sec~\ref{sec3:gradbnd}, Algorithm~\ref{alg:fast-lip}, \textbf{Fast-Lip}). Unlike the per-layer operator-norm-based lower bounds which are often very loose (close to 0, as verified in our experiments) for deep networks, our bounds are much closer to the upper bound given by the best adversarial examples, and thus can be used to evaluate the robustness of DNNs with theoretical guarantee.

\item We show that our proposed method is at least \textit{four orders of magnitude faster} than finding the exact minimum distortion (with Reluplex), and also around \textit{two orders of magnitude (or more) faster} than linear programming (LP) based methods. We can compute a reasonable robustness lower bound within a minute for a ReLU network with up to 7 layers or over ten thousands neurons, which is so far the best available result in the literature to our best knowledge.

\item We show that there is no polynomial time algorithm that can find a lower bound of minimum $\ell_1$ adversarial distortion with a $(1-o(1))\ln n$ approximation ratio (where $n$ is the total number of neurons) unless $\mathsf{NP}$=$\mathsf{P}$ (see Theorem~\ref{thm:intro_hardness_np_p}).
\end{itemize}



\section{Background and related work}\label{sec:related_work}

\subsection{Solving the minimum adversarial distortion}
\label{sec:bg_mindist}
For ReLU networks, the verification problem can be transformed into a Mixed Integer Linear Programming (MILP) problem \cite{lomuscio2017approach,cheng2017maximum,fischetti2017deep} by using binary variables to encode the states of ReLU activation in each neuron. \cite{katz2017reluplex} proposed Reluplex based on satisfiable modulo theory, which encodes the network into a set of linear constraints with special rules to handle ReLU activations and splits the problem into two LP problems based on a ReLU's activation status on demand. Similarly, \cite{ehlers2017formal} proposed Planet, another splitting-based approach using satisfiability (SAT) solvers. These approaches guarantee to find the exact minimum distortion of an adversarial example, and can be used for formal verification. However, due to NP-hard nature of the underlying problem, these approaches only work on very small networks. For example, in \cite{katz2017reluplex}, verifying a feed-forward network with 5 inputs, 5 outputs and 300 total hidden neurons on a single data point can take a few hours. Additionally, Reluplex can find the minimum distortion only in terms of $\ell_{\infty}$ norm ($\ell_1$ is possible via an extension) and cannot easily generalize to $\ell_p$ norm.

\subsection{Computing lower bounds of minimum distortion}
\cite{szegedy2013intriguing} gives a lower bound on the minimum distortion in ReLU networks by computing the product of weight matrices operator norms, but this bound is usually too loose to be useful in practice, as pointed out in \cite{hein2017formal} and verified in our experiments (see Table~\ref{tb:smallnetwork}). A tighter bound was given by \cite{hein2017formal} using local Lipschitz constant on a network with one hidden layer, but their approach requires the network to be continuously-differentiable, and thus cannot be directly applied to ReLU networks. \cite{weng2017evaluating} further provide the lower bound guarantee to non-differentiable functions by Lipschitz continuity assumption and propose the first robustness score, CLEVER, that can evaluate the robustness of DNNs and scale to large ImageNet networks. As also shown in our experiments in Section \ref{sec:exp}, the CLEVER score is indeed a good robustness estimate close to the true minimum distortion given by Reluplex, albeit without providing certificates. Recently, \cite{zico17convex} propose a convex relaxation on the MILP verification problem discussed in Sec~\ref{sec:bg_mindist}, which reduces MILP to LP when the adversarial distortion is in $\ell_{\infty}$ norm. They focus on adversarial training, and compute layer-wise bounds by looking into the dual LP problem.




\subsection{Hardness and approximation algorithms}
$\mathsf{NP}\neq\mathsf{P}$ is the most important and popular assumption in computational complexity in the last several decades. It can be used to show that the decision of the exact case of a problem is hard. However, in several cases, solving one problem approximately is much easier than solving it exactly. For example, there is no polynomial time algorithm to solve the $\mathsf{MAX}$-$\mathsf{CUT}$ problem, but there is a simple $0.5$-approximation polynomial time algorithm. Previous works \cite{katz2017reluplex,sinha2017certifiable} show that there is no polynomial time algorithm to find the minimum adversarial distortion $r_0$ exactly. A natural question to ask is: does there exist a polynomial time algorithm to solve the robustness problem approximately? In other words, can we give a lower bound of $r_0$ with a guaranteed approximation ratio?

From another perspective, $\mathsf{NP}\neq \mathsf{P}$ only rules out the polynomial running time. Some problems might not even have a sub-exponential time algorithm. To rule out that, the most well-known assumption used is the ``Exponential Time Hypothesis''~\cite{ipz98}. The hypothesis states that $\mathsf{3SAT}$ cannot be solved in sub-exponential time in the worst case. Another example is that while tensor rank calculation is NP-hard~\cite{h90}, a recent work \cite{swz17b} proved that there is no $2^{o(n^{1-o(1)})}$ time algorithm to give a constant approximation of the rank of the tensor. There are also some stronger versions of the hypothesis than ETH, e.g., Strong ETH \cite{ip01}, Gap ETH \cite{d16,mr16}, and average case ETH \cite{f02,rsw16}.


\newcommand{\x}{\bm{x}}
\newcommand{\xo}{\bm{x_0}}
\newcommand{\W}[1]{\mathbf{W}^{#1}}
\newcommand{\A}[1]{\mathbf{A}^{#1}}
\newcommand{\DD}[1]{\mathbf{D}^{#1}}
\newcommand{\Lam}[1]{\bm{\Lambda}^{#1}}
\newcommand{\upbias}[1]{\mathbf{T}^{#1}}
\newcommand{\lwbias}[1]{\mathbf{H}^{#1}}
\newcommand{\upbnd}[1]{\bm{u}^{#1}}
\newcommand{\lwbnd}[1]{\bm{l}^{#1}}
\newcommand{\z}{\bm{z}}
\newcommand{\y}{\bm{y}}
\newcommand{\bias}[1]{\bm{b}^{#1}}
\newcommand{\setA}{\mathcal{A}}
\newcommand{\setIpos}[1]{\mathcal{I}^{+}_{#1}}
\newcommand{\setIneg}[1]{\mathcal{I}^{-}_{#1}}
\newcommand{\setIuns}[1]{\mathcal{I}_{#1}}
\newcommand{\set}[1]{\mathcal{#1}}
\newcommand{\Lipsloc}{L_{q,x_0}^j}
   
\section{Robustness guarantees for ReLU networks}\label{sec:result}

\paragraph{Overview of our results.}
We begin with a motivating theorem in Sec~\ref{sec3:hard} showing that there does NOT exist a polynomial time algorithm able to find the minimum adversarial distortion with a $(1 - o(1))\ln n$ approximation ratio. We then introduce notations in Sec~\ref{sec3:relu} and state our main results in Sec~\ref{sec3:convexbnd} and \ref{sec3:gradbnd}, where we develop two approaches that guarantee to obtain a lower bound of minimum adversarial distortion. In Sec~\ref{sec3:convexbnd}, we first demonstrate a general approach to \textit{directly} derive the output bounds of a ReLU network with linear approximations when inputs are perturbed by a general $\ell_p$ norm noise. The analytic output bounds allow us to develop a fast algorithm \fastlin to compute certified lower bound. 
In Sec~\ref{sec3:gradbnd}, we present \fastlip to obtain a certified lower bound of minimum distortion by deriving upper bounds for the local Lipschitz constant. Both methods are highly efficient and allow fast computation of certified lower bounds on large ReLU networks.

\subsection{Finding the minimum distortion with a $0.99\ln n$ approximation ratio is hard}
\label{sec3:hard}
\cite{katz2017reluplex} shows that verifying robustness for ReLU networks is NP-complete; in other words, there is no efficient (polynomial time) algorithm to find the exact minimum adversarial distortion. Here, we further show that even \textit{approximately} finding the minimum adversarial distortion with a guaranteed approximation ratio can be hard. Suppose the $\ell_p$ norm of the true minimum adversarial distortion is $r_0$, 
and a robustness verification program \textsf{A} gives a guarantee that no adversarial examples exist within an $\ell_p$ ball of radius $r$ ($r$ is a lower bound of $r_0$). The approximation ratio $\alpha \coloneqq \frac{r_0}{r} > 1$. We hope that $\alpha$ is close to 1 with a guarantee; for example, if $\alpha$ is a constant regardless of the scale of the network, we can always be sure that $r_0$ is at most $\alpha$ times as large as the lower bound $r$ found by $\textsf{A}$.
Here we relax this requirement and allow the approximation ratio to increase with the number of neurons $n$. In other words, when $n$ is larger, the approximation becomes more inaccurate, but this ``inaccuracy'' can be bounded. However, the following theorem shows that no efficient algorithms exist to give a $0.99\ln n$ approximation in the special case of $\ell_1$ robustness:

\begin{theorem}\label{thm:intro_hardness_np_p}
Unless $\mathsf{P}=\mathsf{NP}$, there is no polynomial time algorithm that gives $(1-o(1))\ln n$-approximation to the $\ell_1$ {\rm ReLU} robustness verification problem with $n$ neurons.
\end{theorem}

Our proof is based on a well-known in-approximability result of \textsf{SET-COVER} problem \cite{rs97,ams06,ds14} and a novel reduction from \textsf{SET-COVER} to our problem. We defer the proof into Appendix~\ref{app:hardness}. The formal definition of the $\ell_1$ ReLU robustness verification problem can be found in Definition~\ref{def:robust_net_real}. Theorem~\ref{thm:intro_hardness_np_p} implies that any efficient (polynomial time) algorithm cannot give better than $(1-o(1))\ln n$-approximation guarantee. Moreover, by making a stronger assumption of Exponential Time Hypothesis ($\mathsf{ETH}$), we can state an explicit result about running time using existing results from \textsf{SET-COVER}~\cite{m12,moshkovitz2012projection},
\begin{corollary}\label{cor:intro_hardness_eth_pgc}
Under $\mathsf{ETH}$, there is no $2^{o(n^c)}$ time algorithm that gives $(1-o(1))\ln n$-approximation to the $\ell_1$ {\rm ReLU} robustness verification problem with $n$ neurons, where $c\in (0,1)$ is some fixed constant.
\end{corollary}

\subsection{ReLU Networks and Their Activation Patterns}
\label{sec3:relu}
Let $\x \in \R^{n_0}$ be the input vector for an $m$-layer neural network with $m-1$ hidden layers and let the number of neurons in each layer be $n_k, \forall k \in [m]$. We use $[n]$ to denote set $\{1,2,\cdots,n\}$. The weight matrix $\W{(k)}$ and bias vector $\bias{(k)}$ for the $k$-th layer have dimension $n_k \times n_{k-1}$ and $n_k$, respectively. Let $\phi_k: \R^{n_0}\to\R^{n_k}$ be the operator mapping from input layer to layer $k$ and $\sigma(\y)$ be the coordinate-wise activation function; for each $k\in [m-1]$, the relation between layer $k-1$ and layer $k$ can be written as $\phi_k(\x) = \sigma(\W{(k)}\phi_{k-1}(\x)+\bias{(k)}),$
where $\W{(k)} \in \R^{n_k \times n_{k-1}}, \bias{(k)} \in \R^{n_k}$. For the input layer and the output layer, we have $\phi_0(\x) = \x$ and $\phi_m(\x) = \W{(m)}\phi_{m-1}(\x)+\bias{(m)}$. The output of the neural network is $f(\x) = \phi_m(\x)$, which is a vector of length $n_m$, and the $j$-th output is its $j$-th coordinate, denoted as $f_j(\x) = [\phi_m(\x)]_j$. For ReLU activation, the activation function $\sigma(\y) = \max(\y,\bm{0})$ is an element-wise operation on the input vector $\y$.

Given an input data point $\xo \in \R^{n_0}$ and a bounded $\ell_p$-norm perturbation $\epsilon \in \R_{+}$, the input $\x$ is constrained in an $\ell_p$ ball $B_p(\xo,\epsilon) := \{ \x ~|~ \| \x - \xo \|_{p} \leq \epsilon \}$. With all possible perturbations in $B_p(\xo,\epsilon)$, the pre-ReLU activation of each neuron has a lower and upper bound $l \in \R$ and $u \in \R$, where $l \leq u$. Let us use $\lwbnd{(k)}_r$ and $\upbnd{(k)}_r$ to denote the lower and upper bound for the $r$-th neuron in the $k$-th layer, and let $\z^{(k)}_r$ be its pre-ReLU activation, where  $\z^{(k)}_r = \W{(k)}_{r,:}\phi_{k-1}(\x)+\bias{(k)}_r$, $\lwbnd{(k)}_r \leq \z^{(k)}_r \leq \upbnd{(k)}_r$, and $\W{(k)}_{r,:}$ is the $r$-th row of $\W{(k)}$. There are three categories of possible activation patterns -- (i) the neuron is always activated: $\setIpos{k} \coloneqq \{ r \in [n_k] ~|~ \upbnd{(k)}_r \geq \lwbnd{(k)}_r \geq 0 \}$, (ii) the neuron is always inactivated: $\setIneg{k} \coloneqq \{ r \in [n_k] ~|~ \lwbnd{(k)}_r \leq \upbnd{(k)}_r \leq 0 \} $, and (iii) the neuron could be either activated or inactivated: $\setIuns{k} \coloneqq \{ r \in [n_k] ~|~ {\lwbnd{(k)}_r < 0 <  \upbnd{(k)}_r} \}$. Obviously, $\{ \setIpos{k}, \setIneg{k},\setIuns{k} \}$ is a partition of set $[n_k]$.




\subsection{Approach 1 (Fast-Lin): Certified lower bounds via linear approximations}
\label{sec3:convexbnd}
\subsubsection{Derivation of the output bounds via linear upper and lower bounds for ReLU}
\label{sec:fastlin_derivation}
In this section, we propose a methodology to \textit{directly} derive upper bounds and lower bounds of the output of an $m$-layer feed-forward ReLU network. The central idea is to derive an \textit{explicit} upper/lower bound based on the linear approximations for the neurons in category (iii) and the signs of the weights associated with the activations. 


\begin{figure*}[t]
\ifdef{\useicmlformat}{
\includegraphics[width=0.7\textwidth]{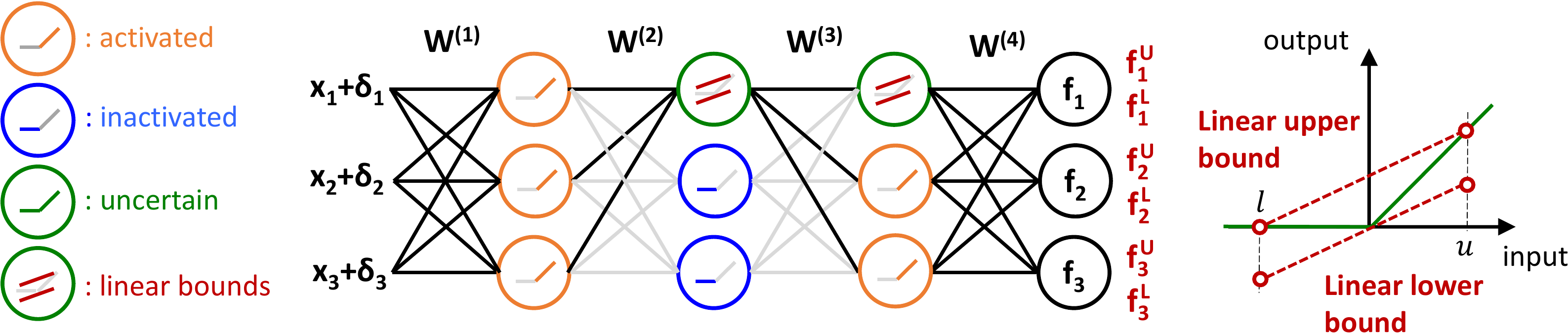}
}{
\includegraphics[width=\textwidth]{figure_relu_network}
}
\centering
\caption{Illustration of deriving output bounds for ReLU networks in Section~\ref{sec3:convexbnd}. The final output upper bounds ($f_j^U$) and lower bounds ($f_j^L$) can be derived by considering the activation status of the neurons with input perturbation $\| \delta \|_p \leq \epsilon$. For neurons in $\setIpos{k}$, their outputs are identical to their inputs; for neurons in $\setIneg{k}$, they can be removed during computation as their outputs are always zero; for neurons in $\setIuns{k}$, their outputs can be bounded by corresponding linear upper bounds and lower bounds considering the signs of associated weights.}
\label{fig:activation}
\end{figure*}


We start with a 2-layers network and then extend it to $m$ layers. The $j$-th output of a 2-layer network is: 
$$f_j(\x) = \sum_{r\in \setIpos{1}, \setIneg{1}, \setIuns{1}} \W{(2)}_{j,r} \sigma(\W{(1)}_{r,:} \x + \bias{(1)}_r) + \bias{(2)}_j.$$For neurons $r \in \setIpos{1}$, we have $\sigma(\W{(1)}_{r,:} \x + \bias{(1)}_r) = \W{(1)}_{r,:} \x + \bias{(1)}_r$; for neurons $r \in \setIneg{1}$, we have $\sigma(\W{(1)}_{r,:} \x + \bias{(1)}_r) = 0.$ For the neurons in category (iii), we propose to use the following linear upper bound and a linear lower bound to replace the ReLU activation $\sigma(y)$:
\begin{equation}
\label{eq:our_cvx_approx}
	\frac{u}{u-l} y \leq \sigma(y) \leq \frac{u}{u-l} (y-l).
\end{equation}
Let $\bm{d}^{(1)}_r := \frac{\upbnd{(1)}_r}{\upbnd{(1)}_r-\lwbnd{(1)}_r}$, we have
\begin{align}
\label{eq:2-layer-cvx-ours}
	 \bm{d}_r^{(1)}(\W{(1)}_{r,:} \x + \bias{(1)}_r) & \leq \; \sigma(\W{(1)}_{r,:} \x + \bias{(1)}_r)
\ifdef{\useicmlformat}{\\ & }{}
     \leq \bm{d}^{(1)}_r  (\W{(1)}_{r,:} \x + \bias{(1)}_r - \lwbnd{(1)}_r)
\ifdef{\useicmlformat}{. \nonumber}{}
\end{align}
To obtain an upper bound and lower bound of $f_j(\x)$ with~\eqref{eq:our_cvx_approx}, set $\bm{d}^{(1)}_r=1$ for $r \in \setIpos{1}$, and we have
\begin{align}
\label{eq:2-layer-newfju}
f_j^U(\x) & = \sum_{r\in \setIpos{1}, \setIuns{1}} \W{(2)}_{j,r} \bm{d}^{(1)}_r (\W{(1)}_{r,:} \x + \bias{(1)}_r)
\ifdef{\useicmlformat}{\\ & }{}
- \sum_{r\in \setIuns{1}, \W{(2)}_{j,r}>0} \W{(2)}_{j,r} \bm{d}^{(1)}_r \lwbnd{(1)}_r + \bias{(2)}_j
\ifdef{\useicmlformat}{,  \nonumber}{}
\end{align}
\vspace{-1em}
\begin{align}
\label{eq:2-layer-newfjl}
f_j^L(\x) & = \sum_{r\in \setIpos{1}, \setIuns{1}} \W{(2)}_{j,r} \bm{d}^{(1)}_r (\W{(1)}_{r,:} \x + \bias{(1)}_r)
\ifdef{\useicmlformat}{\\ & }{}
- \sum_{r\in \setIuns{1}, \W{(2)}_{j,r}<0} \W{(2)}_{j,r} \bm{d}^{(1)}_r \lwbnd{(1)}_r + \bias{(2)}_j
\ifdef{\useicmlformat}{,  \nonumber}{}
\end{align}
where $f_j^L(\x) \leq f_j(\x) \leq f_j^U(\x)$. 
To obtain $f_j^U(\x)$, we take the upper bound of $\sigma(\W{(1)}_{r,:} \x + \bias{(1)}_r)$ for $r \in \setIuns{1}, \W{(2)}_{j,r}>0$ and its lower bound for $r \in \setIuns{1}, \W{(2)}_{j,r} \leq 0$. Both cases share a common term of $\bm{d}^{(1)}_r (\W{(1)}_{r,:} \x + \bias{(1)}_r)$, which is combined into the first summation term in~\eqref{eq:2-layer-newfju} with $r \in \setIuns{1}$. Similarly we get the bound for $f_j^L(\x)$.

For a general $m$-layer ReLU network with the linear approximation~\eqref{eq:our_cvx_approx}, we will show in Theorem~\ref{thm:cvx_bnd} that the network output can be bounded by two explicit functions when the input $\x$ is perturbed with a $\epsilon$-bounded $\ell_p$ noise. We start by defining the activation matrix $\DD{(k)}$ and the additional equivalent bias terms $\upbias{(k)}$ and $\lwbias{(k)}$ for the $k$-th layer in Definition \ref{def:A_eta_tau} and the two explicit functions in \ref{def:f_L_f_U}.   

\begin{definition}[$\A{(k)}, \upbias{(k)}, \lwbias{(k)}$]\label{def:A_eta_tau}
Given matrices $\W{(k)} \in \R^{n_k \times n_{k-1}}$ and vectors $\bias{(k)} \in \R^{n_k}, \forall k \in [m]$.
We define ${\bf D}^{(0)} \in \R^{n_0 \times n_0}$ as an identity matrix. For each $k \in [m-1]$, we define matrix ${\bf D}^{(k)} \in \R^{n_k \times n_k}$ as follows
\begin{align}\label{eq:def_D}
\DD{(k)}_{r,r} & =   
    \begin{cases}
    \frac{\upbnd{(k)}_r}{\upbnd{(k)}_r-\lwbnd{(k)}_r} & \mathrm{~if~} r \in {\cal I}_k; \\
    1 & \mathrm{~if~} r \in {\cal I}_k^+; \\
    0 & \mathrm{~if~} r \in {\cal I}_k^-.  
  \end{cases}
\end{align}
We define matrix $\A{(m-1)} \in \R^{n_m \times n_{m-1}}$ to be $\W{(m)} \DD{(m-1)}$, and for each $k \in \{m-1, m-2, \cdots, 1\}$, matrix $\A{(k-1)} \in \R^{n_m \times n_{k-1}}$ is defined recursively as $\A{(k-1)} = \A{(k)} \W{(k)} \DD{(k-1)}.$
For each $k \in [m-1]$, we define matrices $\upbias{(k)} , \lwbias{(k)} \in \R^{n_k \times n_m}$, where
\begin{align*}
\upbias{(k)}_{r,j} & =   
    \begin{cases}
    \lwbnd{(k)}_r & \mathrm{~if~} r \in \setIuns{k}, \, \A{(k)}_{j,r} > 0 ; \\
    0 & \mathrm{~otherwise~} .
  \end{cases}
  \\
      \lwbias{(k)}_{r,j} & =   
    \begin{cases}
    \lwbnd{(k)}_r & \mathrm{~if~} r \in \setIuns{k}, \, \A{(k)}_{j,r} < 0 ; \\
    0 & \mathrm{~otherwise~} .
  \end{cases}
\end{align*}
\end{definition}

\begin{definition}[Two explicit functions : $f^U(\cdot)$ and $f^L(\cdot)$]\label{def:f_L_f_U}
Let matrices $\A{(k)}$, $\upbias{(k)}$ and $\lwbias{(k)}$ be defined as in Definition~\ref{def:A_eta_tau}. We define two functions $f^U , f^L : \R^{n_0} \rightarrow \R^{n_m}$ as follows. For each input vector $\x \in \R^{n_0}$,
\begin{align*}
f^{U}_j(\x) = & ~ \A{(0)}_{j,:}\x+ \bias{(m)}_j+\sum_{k=1}^{m-1}\A{(k)}_{j,:}(\bias{(k)}-\upbias{(k)}_{:,j}), \\
f^{L}_j(\x) = & ~ \A{(0)}_{j,:}\x+ \bias{(m)}_j+\sum_{k=1}^{m-1}\A{(k)}_{j,:}(\bias{(k)}-\lwbias{(k)}_{:,j}).
\end{align*}
\end{definition}

Now, we are ready to state our main theorem,
\begin{theorem}[Explicit upper and lower bounds]\label{thm:approach_1}
\label{thm:cvx_bnd}
Given an $m$-layer {\rm ReLU} neural network function $f : \R^{n_0} \rightarrow \R^{n_m}$, there exists two explicit functions $f^L : \R^{n_0} \rightarrow \R^{n_m}$ and $f^U :\R^{n_0} \rightarrow \R^{n_m}$ (see Definition~\ref{def:f_L_f_U}) such that $\forall j \in [n_m], \; f_{j}^{L}(\x) \leq f_{j}(\x) \leq f_{j}^{U}(\x), \; \forall \x \in B_p(\xo,\epsilon)$. 
\end{theorem}
The proof of Theorem \ref{thm:cvx_bnd} is in Appendix~\ref{app:approach1_explicit_function}. Since the input $\x \in B_p(\xo,\epsilon)$, we can maximize \eqref{eq:2-layer-newfju} and minimize \eqref{eq:2-layer-newfjl} within this set to obtain a global upper and lower bound of $f_j(\x)$, which has analytical solutions for any $1 \leq p \leq \infty$ and the result is formally shown in Corollary \ref{cor:cvx_bnd} (proof in Appendix~\ref{app:approach1_fixed_value}). In other words, we have \textit{analytic} bounds that can be computed efficiently without resorting to any optimization solvers for general $\ell_p$ distortion, and this is the key to enable fast computation for layer-wise output bounds.

We first formally define the global upper bound $\gamma_j^U$ and lower bound $\gamma_j^L$ of $f_j(\x)$, and then obtain Corollary~\ref{cor:cvx_bnd}.
\begin{definition}[$\gamma_j^L, \gamma_j^U$]\label{def:gamma_j_L_gamma_j_U}
Given a point $\xo \in \R^{n_0}$, a neural network function $f : \R^{n_0} \rightarrow \R^{n_m}$, parameters $p,\epsilon$. Let matrices $\A{(k)}$, $\upbias{(k)}$ and $\lwbias{(k)}$, $\forall k \in [m-1]$ be defined as in Definition~\ref{def:A_eta_tau}. We define $\gamma_j^L, \gamma_j^U, \, \forall j \in [n_m]$ as
\begin{align*}
	\gamma^L_j = \mu_j^- + \nu_j - \epsilon \|\A{(0)}_{j,:}\|_q   \mathrm{~and~}
	\gamma^U_j = \mu_j^+ + \nu_j + \epsilon \|\A{(0)}_{j,:}\|_q ,
\end{align*}
where $1/p+1/q=1$ and $\nu_j, \mu_j^+, \mu_j^-$ are defined as
\begin{align}
 \mu_j^+ = ~ - \sum_{k=1}^{m-1}&\A{(k)}_{j,:}\upbias{(k)}_{:,j} ,\quad
 \mu_j^- = ~ - \sum_{k=1}^{m-1}\A{(k)}_{j,:}\lwbias{(k)}_{:,j} \label{eq:def_mu_j_plus_minus} \\
 \nu_j = ~  &\A{(0)}_{j,:} \xo +  \bias{(m)}_j + \sum_{k=1}^{m-1}\A{(k)}_{j,:}\bias{(k)} \label{eq:def_nu_j}
\end{align}
\end{definition}
\begin{corollary}[Two side bounds in closed-form]
\label{cor:cvx_bnd}
Given a point $\xo \in \R^{n_0}$, an $m$-layer neural network function $f : \R^{n_0} \rightarrow \R^{n_m}$, parameters $p$ and $\epsilon$. For each $j\in [n_m]$, there exist two fixed values $\gamma^L_j$ and $\gamma^U_j$ (see Definition~\ref{def:gamma_j_L_gamma_j_U}) such that $\gamma^L_j \leq f_j (\x) \leq \gamma^U_j, \; \forall \x \in B_p(\x_0, \epsilon ).$
\end{corollary}

\ifdef{\arxivversion}{
\begin{algorithm}[H]
  \caption{\textbf{Fast} Bounding via \textbf{Lin}ear Upper/Lower Bounds for ReLU (\textbf{Fast-Lin})}
  \label{alg:fast-lin}
\begin{algorithmic}[1]
\Require weights and biases of $m$ layers: $\W{(1)}, \cdots, \W{(m)}$, $\bias{(1)}, \cdots, \bias{(m)}$, original class $c$, target class $j$
\Procedure{\textsc{Fast-Lin}}{${\bm x}_0, p, \eps_0$}
\State Replace the last layer weights $\W{(m)}$ with a row vector $\bm{\bar{w}} \leftarrow \W{(m)}_{c,:} - \W{(m)}_{j,:}$ (see Section~\ref{sec:cal_last_layer})
\State Initial $\eps \leftarrow \eps_0$
\While{$\eps$ has not achieved a desired accuracy and iteration limit has not reached}
\State $\lwbnd{(0)}, \upbnd{(0)} \leftarrow \text{don't care}$
\For{$k \leftarrow 1$ {\bfseries to} $m$} \Comment{Compute lower and upper bounds for ReLU unis for all $m$ layers}
\State $\lwbnd{(k)}, \upbnd{(k)} \leftarrow $\textsc{ComputeTwoSideBounds}(${\bm x}_0, \epsilon,p, \lwbnd{(1:k-1)}, \upbnd{(1:k-1)}, k$)
\EndFor
\If {$\lwbnd{(m)} > 0$}  \Comment $\lwbnd{(m)}$ is a scalar since the last layer weight is a row vector
\State $\eps$ is a lower bound; increase $\eps$ using a binary search procedure
\Else
\State $\eps$ is not a lower bound; decrease $\eps$ using a binary search procedure
\EndIf
\EndWhile
\State $\tilde \epsilon_j \leftarrow \eps$
\State \Return $\tilde \epsilon_j$ \Comment $\tilde \epsilon_j$ is a certified lower bound $\beta_L$
\EndProcedure
\Procedure{\textsc{ComputeTwoSideBounds}}{${\bm x}_0, \epsilon,p, \lwbnd{(1:m'-1)}, \upbnd{(1:m'-1)}, m^\prime$}
  \State \Comment{ ${\bm x}_0 \in \R^{n_0}$ : input data vector, $p$ : $\ell_p$ norm, $\epsilon$ : maximum $\ell_p$-norm perturbation}
  \State \Comment{ $\lwbnd{(k)}, \upbnd{(k)}, \, k \in [m^\prime]$ : layer-wise bounds }
  \If{$m^\prime = 1$} \Comment{Step 1: Form $\A{}$ matrices}
  \State $\A{(0)} \leftarrow \W{(1)}$ \Comment{First layer bounds do not depend on $\lwbnd{(0)}, \upbnd{(0)}$}
  \Else
  	\For{$k \leftarrow m^\prime-1$ {\bfseries to} $1$}
  		\If{$k = m^\prime-1$} \Comment{Construct $\DD{(m'-1)}, \A{(m'-1)}, \lwbias{(m'-1)}, \upbias{(m'-1)}$}
  			\State Construct diagonal matrix $\DD{(k)} \in \R^{n_k \times n_k}$ using $\lwbnd{(k)}, \upbnd{(k)}$ according to Eq.~\eqref{eq:def_D}.
    			\State $\A{(m'-1)} \leftarrow \W{(m')}\DD{(m'-1)}$
    	\Else \Comment Multiply all saved $\A{(k)}$ by $\A{(m^\prime-1)}$
    		\State $\A{(k)} \leftarrow \A{(m^\prime-1)}\A{(k)}$ \Comment We save $\A{(k)}$ for next function call
    	\EndIf
        \State $\upbias{(k)} \leftarrow \mathbf{0}, \, \lwbias{(k)} \leftarrow \mathbf{0}$ \Comment{Initialize $\upbias{(k)}$ and $\lwbias{(k)}$}
          \For{all $r \in  \setIuns{k}$}
            \For{$j \leftarrow 1$ {\bfseries to} $n_k$}
            	\If{$\A{(k)}_{j,r} > 0$}
    				\State {$\upbias{(k)}_{r,j} \leftarrow \lwbnd{(k)}_r$}
    	   		\Else 
    				\State {$\lwbias{(k)}_{r,j} \leftarrow \lwbnd{(k)}_r$}
    			\EndIf
            \EndFor
  		  \EndFor

  	\EndFor
  \EndIf 
  \For{$j = 1$ {\bfseries to} $n_{m^\prime}$}  \Comment{Step 2: Compute $\gamma^U$ and $\gamma^L$}
  	\State $\nu_j \leftarrow \A{(0)}_{j,:} \xo +  \bias{(m^\prime)}_j, \, \mu_j^+ \leftarrow 0, \, \mu_j^- \leftarrow 0$ \Comment{Initialize $\nu_j, \mu_j^+, \mu_j^-$}
    \For{$k=1$ {\bfseries to} $m^\prime-1$} \Comment This loop is skipped when $m^\prime = 1$
    \State $\mu_j^+ \leftarrow \mu_j^+ - \A{(k)}_{j,:}\upbias{(k)}_{:,j}$, \quad $\mu_j^- \leftarrow \mu_j^- - \A{(k)}_{j,:}\lwbias{(k)}_{:,j}$ \Comment{According to Eq.~\eqref{eq:def_mu_j_plus_minus}}
    \State $\nu_j \leftarrow \nu_j + \A{(k)}_{j,:}\bias{(k)}$ \Comment{According to Eq.~\eqref{eq:def_nu_j}}
  	\EndFor
  	\State \Comment{ $\nu_j,\mu_j^+,\mu_j^-$ satisfy Definition~\ref{def:gamma_j_L_gamma_j_U}}
  	\State $\gamma_j^U \leftarrow \mu_j^+ + \nu_j +\epsilon \|\A{(0)}_{j,:}\|_q $ 
  	\State $\gamma_j^L \leftarrow \mu_j^- + \nu_j -\epsilon \|\A{(0)}_{j,:}\|_q $ \Comment{Definition~\ref{def:gamma_j_L_gamma_j_U}}
    \EndFor
    \State \Return $\gamma^L, \gamma^U$  
\EndProcedure 
\end{algorithmic}
\end{algorithm}
}{}

\subsubsection{Computing pre-{\rm ReLU} activation bounds }
Theorem \ref{thm:cvx_bnd} and Corollary \ref{cor:cvx_bnd} give us a global lower bound $\gamma_j^L$ and upper bound $\gamma_j^U$ of the $j$-th neuron at the $m$-th layer if we know all the pre-ReLU activation bounds $\lwbnd{(k)}$ and $\upbnd{(k)}$, from layer $1$ to $m-1$, as the construction of $\DD{(k)}$, $\lwbias{(k)}$ and $\upbias{(k)}$ requires $\lwbnd{(k)}$ and $\upbnd{(k)}$ (see Definition~\ref{def:A_eta_tau}). Here, we show how this can be done easily and layer-by-layer. We start from $m = 1$ where $\A{(0)} = \W{(1)}, f^U(\x) = f^L(\x) = \A{(0)}\x+\bias{(1)}$. Then, we can apply Corollary \ref{cor:cvx_bnd} to get the output bounds of each neuron and set them as $\lwbnd{(1)}$ and $\upbnd{(1)}$. Then, we can proceed to $m = 2$ with $\lwbnd{(1)}$ and $\upbnd{(1)}$ and compute the output bounds of second layer by Corollary \ref{cor:cvx_bnd} and set them as $\lwbnd{(2)}$ and $\upbnd{(2)}$. Repeating this procedure for all $m-1$ layers, we will get all the $\lwbnd{(k)}$ and $\upbnd{(k)}$ needed to compute the output range of the $m$-th layer. 

Note that when computing $\lwbnd{(k)}$ and $\upbnd{(k)}$, the constructed $\W{(k)} \DD{(k-1)}$ can be saved and reused for bounding the next layer, which facilitates efficient implementations. Moreover, the time complexity of computing the output bounds of an $m$-layer ReLU network with Theorem~\ref{thm:cvx_bnd} and Corollary~\ref{cor:cvx_bnd} is \textit{polynomial} time in contrast to the approaches in \cite{katz2017reluplex} and \cite{lomuscio2017approach} where SMT solvers and MIO solvers have \textit{exponential} time complexity. The major computation cost is to form $\A{(0)}$ for the $m$-th layer, which involves multiplications of layer weights in a \textit{similar cost of forward propagation}. See the ``ComputeTwoSideBounds'' procedure in Algorithm~\ref{alg:fast-lin} in Appendix~\ref{sec:app_algs}.

\subsubsection{Deriving maximum certified lower bounds of minimum adversarial distortion}
\label{sec:cal_last_layer}
Suppose $c$ is the predicted class of the input data point $\xo$ and the  class is $j$. With Theorem~\ref{thm:cvx_bnd}, the maximum possible lower bound for the targeted attacks $\tilde \epsilon_j$ and un-targeted attacks $\tilde \epsilon$ are 
\begin{equation*}
	\tilde \epsilon_j = \max_{\epsilon} \, \epsilon \; \text{s.t.} \; \gamma_c^L(\epsilon) - \gamma_j^U(\epsilon) > 0 \; \text{  and  } \; \tilde \epsilon = \min_{j \neq c} \, \tilde \epsilon_j .
\end{equation*}
Though it is hard to get analytic forms of $\gamma_c^L(\epsilon)$ and $\gamma_j^U(\epsilon)$ in terms of $\epsilon$, fortunately, we can still obtain $\tilde \epsilon_j$ via a binary search. This is because Corollary~\ref{cor:cvx_bnd} allows us to efficiently compute the numerical values of $\gamma_c^L(\epsilon)$ and $\gamma_j^U(\epsilon)$ given $\epsilon$. It is worth noting that we can further improve the bound by considering $g(\x) := f_c(\x) - f_j(\x)$ at the last layer and apply the same procedure to compute the lower bound of $g(\x)$ (denoted as $\tilde \gamma^L$); this can be done easily by redefining the last layer's weights to be a row vector $\bm{\bar{w}} \coloneqq \W{(m)}_{c,:} - \W{(m)}_{j,:}$. The corresponding maximum possible lower bound for the targeted attacks is $\tilde \epsilon_j = \max \epsilon \; \text{s.t.} \; \tilde \gamma^L(\epsilon)> 0$. 
\ifdef{\arxivversion}{Our proposed algorithm, \fastlin, is shown in Algorithm~\ref{alg:fast-lin}.
}
{We list our complete algorithm, \fastlin, in Appendix~\ref{sec:app_algs}.}

\subsubsection{Discussions}
\label{sec:discuss}
We have shown how to derive explicit output bounds of ReLU network (Theorem~\ref{thm:approach_1}) with the proposed linear approximations and obtain analytical certified lower bounds (Corollary~\ref{cor:cvx_bnd}), which is the key of our proposed algorithm \textbf{Fast-Lin}. 
\cite{zico17convex} presents a similar algorithmic result on computing certified bounds, but our framework and theirs are entirely different -- we use direct computation of layer-wise linear upper/lower bounds in Sec \ref{sec3:convexbnd} with binary search on $\epsilon$, while their results is achieved via the lens of dual LP formulation with Newton's method.
Interestingly, when we choose a special set of lower and upper bounds as in~\eqref{eq:2-layer-cvx-ours} and they choose a special dual LP variable in their equation (8), the two different frameworks coincidentally produce the same procedure for computing layer-wise bounds (the ``ComputeTwoSideBounds'' procedure in \fastlin and Algorithm 1 in \cite{zico17convex}).
However, our choice of bounds~\eqref{eq:2-layer-cvx-ours} is due to computation efficiency, while~\cite{zico17convex} gives a quite different justification. We encourage the readers to read Appendix A.3 in their paper on the justifications for this specific selection of dual variables and understand this robustness verification problem from different perspectives.

\newcommand{\gradl}[1]{\tilde{\bm{l}}^{#1}}
\newcommand{\gradu}[1]{\tilde{\bm{u}}^{#1}}

\newcommand{\gradC}[1]{\mathbf{C}^{#1}}
\newcommand{\gradL}[1]{\mathbf{L}^{#1}}
\newcommand{\gradU}[1]{\mathbf{U}^{#1}}
\newcommand{\Y}[1]{\mathbf{Y}^{#1}}

\subsection{Approach 2 (Fast-Lip): Certified lower bounds via bounding the local Lipschitz constant}
\label{sec3:gradbnd}

\cite{weng2017evaluating} shows a non-trivial lower bound of minimum adversarial distortion for an input example $\xo$ in targeted attacks is $\min \left( g(\bm{x_0})/L_{q,x_0}^j,\epsilon \right)$, where $g(\x)=f_c(\x) - f_j(\x), \, L_{q,x_0}^j$ is the local Lipschitz constant of $g(\x)$ in $B_p(\xo, \epsilon)$, $\,j$ is the target class, $c$ is the original class, and $1/p + 1/q = 1$. For un-targeted attacks, the lower bound can be presented in a similar form. \cite{weng2017evaluating} uses sampling techniques to estimate the local Lipschitz constant and compute an estimated lower bound without certificates.

Here, we propose a new algorithm to compute a \textit{certified} lower bound of the minimum adversarial distortion by upper bounding the local Lipschitz constant. To start with, let us rewrite the relations of subsequent layers in the following form: $\phi_k(\x) = \Lam{(k)}(\W{(k)}\phi_{k-1}(\x)+\bias{(k)})$, where $\sigma(\cdot)$ is replaced by the diagonal activation pattern matrix $\Lam{(k)}$ that encodes the status of neurons $r$ in $k$-th layer:
\begin{equation}
	\Lam{(k)}_{r,r} =     
    \begin{cases}
    1 \text{ or } 0 & \text{if $r \in \setIuns{k}$} \\
    1 & \text{if $r \in \setIpos{k}$} \\
    0 & \text{if $r \in \setIneg{k}$}     
  \end{cases}
\end{equation}
and $\Lam{(m)} = \bm{I}_{n_m}$. With a slight abuse of notation, let us define $\Lam{(k)}_a$ as a diagonal activation matrix for neurons in the $k$-th layer who are always \textbf{a}ctivated, i.e. the $r$-th diagonal is $1$ if $r \in \setIpos{k}$ and $0$ otherwise, and  $\Lam{(k)}_u$ as the diagonal activation matrix for $k$-th layer neurons whose status are \textbf{u}ncertain, i.e. the $r$-th diagonal is $1$ or $0$ (to be determined) if $r \in \setIuns{k}$, and $0$ otherwise. Therefore, we have $\Lam{(k)} = \Lam{(k)}_a + \Lam{(k)}_u$. 
We can obtain $\Lam{(k)}$ for $\x \in B_p(\xo,\epsilon)$ by applying Algorithm~\ref{alg:fast-lin} and check the lower and upper bounds for each neuron $r$ in layer $k$.


\subsubsection{A general upper bound of Lipschitz constant in $\ell_q$ norm}
The central idea is to compute upper bounds of $L_{q,x_0}^j$ by exploiting the three categories of activation patterns in ReLU networks when the allowable inputs are in $B_p(\xo,\epsilon)$. $L_{q, \x_0}^j$ can be defined as the maximum norm of directional derivative as shown in \cite{weng2017evaluating}.
For the ReLU network, the maximum directional derivative norm can be found by examining all the possible activation patterns and take the one (the worst-case) that results in the largest gradient norm. However, as all possible activation patterns grow exponentially with the number of the neurons, it is impossible to examine all of them in brute-force. Fortunately, computing the worst-case pattern on each element of $\nabla g(\x)$ (i.e. $[\nabla g(x)]_k, \, k \in [n_0]$) is much easier and more efficient. In addition, we apply a simple fact that the maximum norm of a vector (which is $\nabla g(\x), \x \in B_p(\xo,\epsilon)$ in our case) is upper bounded by the norm of the maximum value for each components. By computing the worst-case pattern on $[\nabla g(\x)]_k$ and its norm, we can obtain an upper bound of the local Lipschitz constant, which results in a certified lower bound of minimum distortion. 

Below, we first show how to derive an upper bound of the Lipschitz constant by computing the worst-case activation pattern on $[\nabla g(\x)]_k$ for $2$ layers. Next, we will show how to apply it repeatedly for a general $m$-layer network, and the algorithm is named \textbf{Fast-Lip}. Note that for simplicity, we will use $[\nabla f_j(\x)]_k$ to illustrate our derivation; however, it is easy to extend to $[\nabla g(\x)]_k$ as $g(\x) = f_c(\x) - f_j(\x)$ by simply replacing last layer weight vector by $\W{(m)}_{c,:}-\W{(m)}_{j,:}$.

\ifdef{\arxivversion}{
\begin{algorithm}[ht]
  \caption{\textbf{Fast} Bounding via Upper Bounding Local \textbf{Lip}schitz Constant (\textbf{Fast-Lip})}
  \label{alg:fast-lip}
\begin{algorithmic}[1]
\Require Weights of $m$ layers: $\W{(1)}, \cdots \W{(m)}$, original class $c$, target class $j$
\Procedure{\textsc{Fast-Lip}}{${\bm x}_0, p, \eps$}
\State Replace the last layer weights $\W{(m)}$ with a row vector $\bm{\bar{w}} \leftarrow \W{(m)}_{c,:} - \W{(m)}_{j,:}$ (see Section~\ref{sec:cal_last_layer})
\State Run \textsc{Fast-Lin} to find layer-wise bounds $\lwbnd{(i)}, \upbnd{(i)}$, and form $\setIpos{i}, \setIneg{i}, \setIuns{i}$ fo all $i \in [m]$
\State $\gradC{(0)} \leftarrow \W{(1)}$,
$\gradL{(0)} \leftarrow \bm{0}$,
$\gradU{(0)} \leftarrow \bm{0}$
\For{$l \leftarrow 1$ to $m-1$}
	\State $\gradC{(l)}, \gradL{(l)}, \gradU{(l)}$ = \textsc{BoundLayerGrad}$(\gradC{(l-1)}, \gradL{(l-1)}, \gradU{(l-1)}, \W{(l+1)}, n_{l+1}, \setIpos{l}, \setIneg{l}, \setIuns{l})$
\EndFor
\State \Comment $\bm{v} \in \R^{n_0}$ because the last layer is replaced with a row vector $\bm{\bar{w}}$
\State $\bm{v} \leftarrow \max(|\gradC{(m-1)} + \gradL{(m-1)}|, |\gradC{(m-1)} + \gradU{(m-1)}|)$ \Comment All operations are element-wise; 
\State $\tilde \eps_j \leftarrow \min(\frac{g(\x_0)}{\| \bm{v} \|_q}, \eps)$ \Comment $q$ is the dual norm of $p$, $\frac{1}{p}+\frac{1}{q}=1$
\State \Return $\tilde \eps_j$ \Comment $\tilde \eps_j$ is a certified lower bound $\beta_L$. We can also bisect $\tilde \eps_j$ (omitted).
\EndProcedure 
\Procedure{\textsc{BoundLayerGrad}}{$\gradC{}, \gradL{}, \gradU{}, \W{}, n^\prime, \setIpos{}, \setIneg{}, \setIuns{}$}
\For{$k \in [n_0]$} \Comment $n_0$ is the dimension of $\x_0$
\For{$j \in [n^\prime]$}
\State $\gradC{^\prime}_{j,k} \leftarrow \sum\limits_{i \in \setIpos{}} \W{}_{j,i} \gradC{}_{i,k}$

\State $\gradU{^\prime}_{j,k} \leftarrow \sum\limits_{i \in \setIpos{}, \W{}_{j,i} > 0} \hspace{-4mm} \W{}_{j,i} \gradU{}_{i,k} \hspace{2mm} + \sum\limits_{i \in \setIpos{}, \W{}_{j,i} < 0} \hspace{-4mm} \W{}_{j,i} \gradL{}_{i,k} \hspace{2mm}+ $
\State $\hspace{13mm} \sum\limits_{i \in \setIuns{}, \W{}_{j,i} < 0, \gradC{}_{i,k} + \gradL{}_{i,k} < 0} \hspace{-3mm}  \W{}_{j,i} (\gradC{}_{i,k} + \gradL{}_{i,k}) \hspace{2mm} + \sum\limits_{i \in \setIuns{}, \W{}_{j,i} > 0, \gradC{}_{i,k} + \gradU{}_{i,k} > 0} \hspace{-3mm}  \W{}_{j,i} (\gradC{}_{i,k} + \gradU{}_{i,k})$

\State $\gradL{^\prime}_{j,k} \leftarrow \sum\limits_{i \in \setIpos{}, \W{}_{j,i} > 0} \hspace{-4mm} \W{}_{j,i} \gradL{}_{i,k} \hspace{2mm} + \sum\limits_{i \in \setIpos{}, \W{}_{j,i} < 0} \hspace{-4mm} \W{}_{j,i} \gradU{}_{i,k} \hspace{2mm} + $
\State $\hspace{13mm} \sum\limits_{i \in \setIuns{}, \W{}_{j,i} > 0, \gradC{}_{i,k} + \gradL{}_{i,k} < 0} \hspace{-3mm}  \W{}_{j,i} (\gradC{}_{i,k} + \gradL{}_{i,k}) \hspace{2mm} + \sum\limits_{i \in \setIuns{}, \W{}_{j,i} < 0, \gradC{}_{i,k} + \gradU{}_{i,k} > 0} \hspace{-3mm}  \W{}_{j,i} (\gradC{}_{i,k} + \gradU{}_{i,k})$

\EndFor
\EndFor
\State \Return $\gradC{^\prime}, \gradL{^\prime}, \gradU{^\prime}$
\EndProcedure
\end{algorithmic}
\end{algorithm}
}{}

\paragraph{Bounds for a 2-layer ReLU Network.} The gradient is:
$$[\nabla f_j(\x)]_k = \W{(2)}_{j,:} \Lam{(1)}_a \W{(1)}_{:,k} + \W{(2)}_{j,:} \Lam{(1)}_u \W{(1)}_{:,k}.$$ The first term $\W{(2)}_{j,:} \Lam{(1)}_a \W{(1)}_{:,k}$ is a constant and all we need to bound is the second term $\W{(2)}_{j,:} \Lam{(1)}_u \W{(1)}_{:,k}$. Let $\gradC{(1)}_{j,k} = \W{(2)}_{j,:} \Lam{(1)}_a \W{(1)}_{:,k},\,$ $\gradL{(1)}_{j,k}$ and $\gradU{(1)}_{j,k}$ be the lower and upper bounds of the second term, we have 
\begin{equation*}
\gradL{(1)}_{j,k} =  \hspace{-4mm} \sum_{i \in \setIuns{1},\W{(2)}_{j,i}\W{(2)}_{i,k} < 0} \hspace{-8mm} \W{(2)}_{j,i}\W{(2)}_{i,k}, \; \gradU{(1)}_{j,k} =  \hspace{-4mm} \sum_{i \in \setIuns{1},\W{(2)}_{j,i}\W{(2)}_{i,k} > 0} \hspace{-8mm}  \W{(2)}_{j,i}\W{(2)}_{i,k}
\end{equation*} 
\[\max_{\x \in B_p(\xo,\epsilon)} |[\nabla f_j(\x)]_k| \leq \max (|\gradC{(1)}_{j,k} + \gradL{(1)}_{j,k}|, |\gradC{(1)}_{j,k} + \gradU{(1)}_{j,k}|).
\]

\paragraph{Bounds for 3 layers or more.} For 3 or more layers, we can apply the above 2-layer results recursively, layer-by-layer. For example, for a 3-layer ReLU network,  $$[\nabla f_j(\x)]_k = \W{(3)}_{j,:} \Lam{(2)} \W{(2)} \Lam{(1)} \W{(1)}_{:,k}, $$
if we let $\Y{(1)}_{:,k} = \W{(2)} \Lam{(1)} \W{(1)}_{:,k}$, then $[\nabla f_j(\x)]_k$ is reduced to the following form that is similar to 2 layers:
\begin{align}
\label{eq:grad_3layer1}
	[\nabla f_j(\x)]_k & = \W{(3)}_{j,:} \Lam{(2)} \Y{(1)}_{:,k} \\
\label{eq:grad_3layer2}
    & = \W{(3)}_{j,:} \Lam{(2)}_a \Y{(1)}_{:,k} + \W{(3)}_{j,:} \Lam{(2)}_u \Y{(1)}_{:,k}
\end{align}
To obtain the bound in \eqref{eq:grad_3layer1}, we first need to obtain a lower bound and upper bound of $\Y{(1)}_{:,k}$, where we can directly apply the 2-layer results to get an upper and an lower bound for each component $i$ as $\gradC{(1)}_{i,k} + \gradL{(1)}_{i,k} \leq \Y{(1)}_{i,k} \leq \gradC{(1)}_{i,k} + \gradU{(1)}_{i,k}$. Next, the first term $\W{(3)}_{j,:} \Lam{(2)}_a \Y{(1)}_{:,k}$ in \eqref{eq:grad_3layer2} can be lower bounded and upper bounded respectively by 
\begin{align}
\label{eq:3-layer-LB1}
& \sum_{i \in \setIpos{2}} \W{(3)}_{j,i} \gradC{(1)}_{i,k} + \hspace{-3mm} \sum_{i \in \setIpos{2}, \W{(3)}_{j,i} > 0} \hspace{-4mm} \W{(3)}_{j,i} \gradL{(1)}_{i,k} + \hspace{-4mm} \sum_{i \in \setIpos{2}, \W{(3)}_{j,i} < 0} \hspace{-3.5mm} \W{(3)}_{j,i} \gradU{(1)}_{i,k} \\
\label{eq:3-layer-UB1}
& \sum_{i \in \setIpos{2}} \W{(3)}_{j,i} \gradC{(1)}_{i,k} + \hspace{-3mm} \sum_{i \in \setIpos{2}, \W{(3)}_{j,i} > 0} \hspace{-4mm} \W{(3)}_{j,i} \gradU{(1)}_{i,k} + \hspace{-4mm} \sum_{i \in \setIpos{2}, \W{(3)}_{j,i} < 0} \hspace{-3.5mm} \W{(3)}_{j,i} \gradL{(1)}_{i,k}
\end{align}
whereas the second term $\W{(3)}_{j,:} \Lam{(2)}_u \Y{(1)}_{:,k}$ in \eqref{eq:grad_3layer2} is bounded by $\sum_{i \in \mathcal{P}} \W{(3)}_{j,i} (\gradC{(1)}_{i,k} + \gradL{(1)}_{i,k}) + \sum_{i \in \mathcal{Q}} \W{(3)}_{j,i} (\gradC{(1)}_{i,k} + \gradU{(1)}_{i,k})$ with lower/upper bound index sets $\mathcal{P}_L,\mathcal{Q}_L$ and $\mathcal{P}_U,\mathcal{Q}_U$:
\begin{align}
\label{eq:3-layer-LB2}
	& \mathcal{P}_L = \{i \mid i \in \setIuns{2}, \W{(3)}_{j,i} > 0, \gradC{(1)}_{i,k} + \gradL{(1)}_{i,k} < 0 \}, \nonumber \\ 
    & \mathcal{Q}_L = \{i \mid i \in \setIuns{2}, \W{(3)}_{j,i} < 0, \gradC{(1)}_{i,k} + \gradU{(1)}_{i,k} > 0 \};\\
    \label{eq:3-layer-UB2}
    & \mathcal{P}_U = \{i \mid i \in \setIuns{2}, \W{(3)}_{j,i} < 0, \gradC{(1)}_{i,k} + \gradL{(1)}_{i,k} < 0 \}, \nonumber \\ 
    & \mathcal{Q}_U = \{i \mid i \in \setIuns{2}, \W{(3)}_{j,i} > 0, \gradC{(1)}_{i,k} + \gradU{(1)}_{i,k} > 0 \}.
\end{align}
Let $\gradC{(2)}_{j,k} = \sum_{i \in \setIpos{2}} \W{(3)}_{j,i} \gradC{(1)}_{i,k}$, $\gradU{(2)}_{j,k}+\gradC{(2)}_{j,k}$ and $\gradL{(2)}_{j,k}+\gradC{(2)}_{j,k}$ be the upper and lower bound of $[\nabla f_j(\x)]_k$, we have 
\begin{equation*}
\gradU{(2)}_{j,k} + \gradC{(2)}_{j,k} = \eqref{eq:3-layer-UB1} + \eqref{eq:3-layer-UB2}  \; \; \text{and} \; \; \gradL{(2)}_{j,k} + \gradC{(2)}_{j,k} = \eqref{eq:3-layer-LB1} + \eqref{eq:3-layer-LB2},
\end{equation*}
\[
\max_{\x \in B_p(\xo,\epsilon)} |[\nabla f_j(\x)]_k| \hspace{-1mm} \leq \hspace{-1mm} \max (|\gradL{(2)}_{j,k}+\gradC{(2)}_{j,k}|, |\gradU{(2)}_{j,k}+\gradC{(2)}_{j,k}|).
\] 
Thus, this technique can be used iteratively to solve $\max_{\x \in B_p(\xo,\epsilon)} |[\nabla f_j(\x)]_k|$ for a general $m$-layer network, and we can easily bound any $q$ norm of $\nabla f_j(\x)$ by the $q$ norm of the vector of maximum values. For example,
\ifdef{\arxivversion}{
\begin{align*}
& \max_{\x \in B_p(\xo,\epsilon)} \| \nabla f_j(\x) \|_1 \leq \sum_j  \max_{\x \in B_p(\xo,\epsilon)} |[\nabla f_j(\x)]_k| , \\
& \max_{\x \in B_p(\xo,\epsilon)} \| \nabla f_j(\x) \|_2 \leq \sqrt{\sum_j ( \max_{\x \in B_p(\xo,\epsilon)} |[\nabla f_j(\x)]_k| )^2 }, \\
& \max_{\x \in B_p(\xo,\epsilon)} \| \nabla f_j(\x) \|_\infty \leq \max_j \max_{\x \in B_p(\xo,\epsilon)} |[\nabla f_j(\x)]_k|.
\end{align*}
}{
\begin{equation*}
\max_{\x \in B_p(\xo,\epsilon)} \| \nabla f_j(\x) \|_q \leq \left( {\sum_k ( \max_{\x \in B_p(\xo,\epsilon)} |[\nabla f_j(\x)]_k| )^q } \right )^{\frac{1}{q}}
\end{equation*}
}
\ifdef{\arxivversion}{
The full procedure of \fastlip is in Algorithm~\ref{alg:fast-lip}.
}{We list our full procedure, \fastlip, in Appendix~\ref{sec:app_algs}.
}

\paragraph{Further speed-up.} Note that in the 3-layer example, we compute the bounds from right to left, i.e. we first get the bound of $\W{(2)} \Lam{(1)} \W{(1)}_{:,k}$, and then bound $\W{(3)}_{j,:} \Lam{(2)} \Y{(1)}_{:,k}$. Similarly, we can compute the bounds from left to right -- get the bound of $\W{(3)}_{j,:} \Lam{(2)} \W{(2)}$ first, and then bound $\Y{(2)}_{j,:} \Lam{(1)} \W{(1)}_{:,k}$, where $\Y{(2)}_{j,:} = \W{(3)}_{j,:} \Lam{(2)} \W{(2)}$. Since the dimension of the output layer ($n_m$) is typically far less than the dimension of the input vector ($n_0$), computing the bounds from left to right is more efficient as the matrix $\Y{}$ has a smaller dimension of $n_m \times n_k$ rather than $n_k \times n_0$.

\begin{table*}[ht]
\centering
\caption{Comparison of methods of computing certified lower bounds (\fastlin, \fastlip, \lp, \lpfull,\opnorm), estimated lower bound (\clever), exact minimum distortion (\reluplex) and upper bounds (\textbf{Attack}: CW for $p = 2, \infty$, EAD for $p = 1$) on (a) 2, 3 layers \textit{toy} MNIST networks with 20 neurons per layer and (b) \textit{large} networks with 2-7 layers, 1024 or 2048 nodes per layer. Differences of lower bounds and speedup are measured on the best bound from our proposed algorithms and \lp-based approaches (the \textbf{bold} numbers in each row). In (a), we show how close our fast bounds are to exact minimum distortions (\textbf{Reluplex}) and the bounds that are slightly tighter but very expensive (\lpfull). In (b), \lpfull and \reluplex are \textit{computationally infeasible} for all the networks reported here.}
\label{tb:smallnetwork_and_large}
\vspace{0.3em}

\begin{subtable}[t]{\textwidth}

\centering
\scalebox{0.72}{
\begin{tabular}{|c|c|c|
>{\centering\arraybackslash}p{9ex}
>{\centering\arraybackslash}p{9ex}|
>{\centering\arraybackslash}p{9ex}
>{\centering\arraybackslash}p{9ex}
||c||c||cc|}
\hline
\multicolumn{3}{|c|}{Toy Networks}                                                                                                   & \multicolumn{8}{c|}{Average Magnitude of Distortions on 100 Images}                                                                        \\ \hline
\multirow{3}{*}{Network}                                                       & \multirow{3}{*}{$p$}   & \multirow{3}{*}{Target}  & \multicolumn{4}{c||}{Certified Lower Bounds}                                      & difference       & Exact                   & \multicolumn{2}{c|}{Uncertified}                    \\ \cline{4-7}\cline{9-11}
                                                                               &                        &                          & \multicolumn{2}{c|}{Our bounds} & \multicolumn{2}{c||}{Our Baselines} & ours vs.        & Reluplex                & CLEVER                    &  Attacks                \\ \cline{4-7}
                                                                               &                        &                          & Fast-Lin     & Fast-Lip         & LP         & LP-Full                      & LP(-Full)         & \cite{katz2017reluplex} & \cite{weng2017evaluating} & CW/EAD                  \\ \hline

\multirow{3}{*}{\begin{tabular}[c]{@{}c@{}}MNIST\\ $2\times[20]$\end{tabular}} & $\infty$               & rand      & \bf 0.0309 & 0.0270     & \bf 0.0319 & 0.0319     &  -3.2\% & 0.07765  & 0.0428 & 0.08060 \\
                                                                               & $2$                    & rand      & \bf 0.6278 & 0.6057     & 0.7560     & \bf 0.9182 & -31.6\% & -        & 0.8426 & 1.19630 \\
                                                                               & $1$                    & rand      & 3.9297     & \bf 4.8561 & 4.2681     & \bf 4.6822 &  +3.7\% & -        & 5.858  & 11.4760 \\ \hline
\multirow{3}{*}{\begin{tabular}[c]{@{}c@{}}MNIST\\ $3\times[20]$\end{tabular}} & $\infty$               & rand      & \bf 0.0229 & 0.0142     & 0.0241     & \bf 0.0246 &  -6.9\% & 0.06824  & 0.0385 & 0.08114 \\
                                                                               & $2$                    & rand      & \bf 0.4652 & 0.3273     & 0.5345     & \bf 0.7096 & -34.4\% & -        & 0.7331 & 1.22570 \\
                                                                               & $1$                    & rand      & \bf 2.8550 & 2.8144     & 3.1000     & \bf 3.5740 & -20.1\% & -        & 4.990  & 10.7220 \\ \hline
\end{tabular}
}
\center
\caption{Toy networks. \reluplex is designed to verify $\ell_\infty$ robustness so we omit its numbers for $p = 2,1$.}
\label{tb:smallnetwork_main}
\end{subtable}

\begin{subtable}[t]{\textwidth}
\centering
\scalebox{0.68}{
\begin{tabular}{|c|c|cc|cc||c||cc||cc|c||c|}
\hline

\multicolumn{2}{|c|}{Large Networks}                                                                        & \multicolumn{7}{c||}{Average Magnitude of Distortion on 100 Images}                                                                                                             & \multicolumn{4}{c|}{Average Running Time per Image}       \\ \hline
\multirow{3}{*}{Network}                                                      & \multirow{3}{*}{$p$}        & \multicolumn{4}{c||}{Certified Bounds}                                                     &  diff            & \multicolumn{2}{c||}{Uncertified}               & \multicolumn{3}{c||}{Certified Bounds}                   & Speedup        \\ \cline{3-6}\cline{8-12}
                                                                              &                             & \multicolumn{2}{c|}{Our bounds} &     LP                  &  Op-norm                       &  ours          & CLEVER                      & Attacks           & \multicolumn{2}{c|}{Our bounds}  &    LP                 & ours         \\ \cline{3-4}\cline{10-11}
                                                                              &                             &     Fast-Lin &   Fast-Lip       &   (Baseline)     &  \cite{szegedy2013intriguing}  &  vs. LP            & \cite{weng2017evaluating}   & CW/EAD            &     Fast-Lin      &     Fast-Lip & (Baseline)   & vs. LP           \\ \hline

\multirow{3}{*}{\begin{tabular}[c]{@{}c@{}}MNIST\\  $2 \times [1024]$\end{tabular}} &    $\infty$           & \bf 0.03083 & 0.02512     & \bf 0.03386 & 0.00263 & -8.9\%  & 0.0708 & 0.1291  & \bf 156 ms & 219 ms     & \bf 20.8 s & 133X  \\
                                                                                    &    $2$                & \bf 0.63299 & 0.59033     & \bf 0.75164 & 0.40201 & -15.8\% & 1.2841 & 1.8779  & \bf 128 ms & 234 ms     & \bf 195 s  & 1523X \\
                                                                                    &    $1$                & 3.88241     & \bf 5.10000 & \bf 4.47158 & 0.35957 & +14.1\% & 7.4186 & 17.259  & 139 ms     & \bf 1.40 s & \bf 48.1 s & 34X   \\ \hline
\multirow{3}{*}{\begin{tabular}[c]{@{}c@{}}MNIST\\  $3 \times [1024]$\end{tabular}} &    $\infty$           & \bf 0.02216 & 0.01236     & \bf 0.02428 & 0.00007 & -8.7\%  & 0.0717 & 0.1484  & \bf 1.12 s & 1.11 s     & \bf 52.7 s & 47X   \\
                                                                                    &    $2$                & \bf 0.43892 & 0.26980     & \bf 0.49715 & 0.10233 & -11.7\% & 1.2441 & 2.0387  & \bf 906 ms & 914 ms     & \bf 714 s  & 788X  \\
                                                                                    &    $1$                & \bf 2.59898 & 2.25950     & \bf 2.91766 & 0.01133 & -10.9\% & 7.2177 & 17.796  & \bf 863 ms & 3.84 s     & \bf 109 s  & 126X  \\ \hline
\multirow{3}{*}{\begin{tabular}[c]{@{}c@{}}MNIST\\  $4 \times [1024]$\end{tabular}} &    $\infty$           & \bf 0.00823 & 0.00264     & -           & 0.00001 & -       & 0.0793 & 0.1303  & \bf 2.25 s & 3.08 s     & -          & -     \\
                                                                                    &    $2$                & \bf 0.18891 & 0.06487     & -           & 0.17734 & -       & 1.4231 & 1.8921  & \bf 2.37 s & 2.72 s     & -          & -     \\
                                                                                    &    $1$                & \bf 1.57649 & 0.72800     & -           & 0.00183 & -       & 8.9764 & 17.200  & \bf 2.42 s & 2.91 s     & -          & -     \\ \hline
\multirow{3}{*}{\begin{tabular}[c]{@{}c@{}}CIFAR\\  $5 \times [2048]$\end{tabular}} &    $\infty$           & \bf 0.00170 & 0.00030     & -           & 0.00000 & -       & 0.0147 & 0.02351 & \bf 26.2 s & 78.1 s     & -          & -     \\
                                                                                    &    $2$                & \bf 0.07654 & 0.01417     & -           & 0.00333 & -       & 0.6399 & 0.9497  & \bf 36.8 s & 49.4 s     & -          & -     \\
                                                                                    &    $1$                & \bf 1.18928 & 0.31984     & -           & 0.00000 & -       & 9.7145 & 21.643  & \bf 37.5 s & 53.6 s     & -          & -     \\ \hline
\multirow{3}{*}{\begin{tabular}[c]{@{}c@{}}CIFAR\\  $6 \times [2048]$\end{tabular}} &    $\infty$           & \bf 0.00090 & 0.00007     & -           & 0.00000 & -       & 0.0131 & 0.01866 & \bf 37.0 s & 119 s      & -          & -     \\
                                                                                    &    $2$                & \bf 0.04129 & 0.00331     & -           & 0.01079 & -       & 0.5860 & 0.7635  & \bf 60.2 s & 95.6 s     & -          & -     \\
                                                                                    &    $1$                & \bf 0.72178 & 0.08212     & -           & 0.00000 & -       & 8.2507 & 17.160  & \bf 61.4 s & 88.2 s     & -          & -     \\ \hline
\multirow{3}{*}{\begin{tabular}[c]{@{}c@{}}CIFAR\\  $7 \times [1024]$\end{tabular}} &    $\infty$           & \bf 0.00134 & 0.00008     & -           & 0.00000 & -       & 0.0112 & 0.0218  & \bf 10.6 s & 29.2 s     & -          & -     \\
                                                                                    &    $2$                & \bf 0.05938 & 0.00407     & -           & 0.00029 & -       & 0.5145 & 0.9730  & \bf 16.9 s & 27.3 s     & -          & -     \\
                                                                                    &    $1$                & \bf 0.86467 & 0.09239     & -           & 0.00000 & -       & 8.630  & 22.180  & \bf 17.6 s & 26.7 s     & -          & -     \\ \hline
 \end{tabular}
}
\caption{Larger networks. ``\textbf{-}'' indicates the corresponding method is computationally infeasible for that network.}
\label{tb:largenetwork_main}
\end{subtable}
\vspace{-2em}
\end{table*}

\begin{table*}[h!]
\centering
\caption{Comparison of the lower bounds for $\ell_\infty$ distortion found by our  algorithms on models with defensive distillation (DD)~\cite{papernot2016distillation} with temperature = 100 and adversarial training~\cite{madry2017towards} with $\epsilon = 0.3$ for three targeted attack classes.}
\label{tb:distill}
\vspace{0.3em}
\ifdef{\useicmlformat}
{
}{
}
\resizebox{\linewidth}{!}{
\begin{tabular}{|c|c|ccc|ccc|ccc|}
\cline{1-11}
                             & \multicolumn{1}{c|}{}       & \multicolumn{3}{c|}{runner-up target}        & \multicolumn{3}{c|}{random target}           & \multicolumn{3}{c|}{least-likely target}                     \\ \hline
\multicolumn{1}{|c|}{Network} & \multicolumn{1}{c|}{Method} & Undefended &    DD    & Adv. Training & Undefended &    DD    & Adv. Training & Undefended &    DD    & Adv. Training                \\ \hline

\multirow{2}{*}{\makecell{MNIST \\ 3*[1024]}} & Fast-Lin & 0.01826 & 0.02724 & \bf 0.14730 & 0.02211 & 0.03827 & \bf 0.17275 & 0.02427 & 0.04967 & \bf 0.20136 \\
                                              & Fast-Lip & 0.00965 & 0.01803 & 0.09687     & 0.01217 & 0.02493 & 0.11618     & 0.01377 & 0.03207 & 0.13858     \\ \hline
\multirow{2}{*}{\makecell{MNIST \\ 4*[1024]}} & Fast-Lin & 0.00715 & 0.01561 & \bf 0.09579 & 0.00822 & 0.02045 & \bf 0.11209 & 0.00898 & 0.02368 & \bf 0.12901 \\
                                              & Fast-Lip & 0.00087 & 0.00585 & 0.04133     & 0.00145 & 0.00777 & 0.05048     & 0.00183 & 0.00903 & 0.06015     \\ \hline

\end{tabular}}
\end{table*}

\section{Experiments}\label{sec:exp}

In this section, we perform extensive experiments to evaluate the performance of our proposed two lower-bound based robustness certificates on networks with different sizes and with different defending techniques during training process. Specifically, we compare our proposed bounds\footnote{\url{https://github.com/huanzhang12/CertifiedReLURobustness}} (\fastlin, \fastlip) with Linear Programming (LP) based methods (\lp, \lpfull), formal verification methods (\reluplex), lower bound by global Lipschitz constant (\opnorm), estimated lower bounds (\clever) and attack algorithms (\textbf{Attacks}) for toy networks (2-3 layers with 20 neurons in each layer) and large networks (2-7 layers with 1024 or 2048 neurons in each layer) in Table~\ref{tb:smallnetwork_and_large}. The evaluation on the effects of defending techniques is presented in Table~\ref{tb:distill}. All bound numbers are the average of 100 random test images with random attack targets, and running time (per image) for all methods is measured on a single CPU core.
We include detailed setup of experiments, descriptions of each method, additional experiments and discussions in Appendix~\ref{app:exp} (See Tables~\ref{tb:smallnetwork} and \ref{tb:largenetwork_app}).
The results suggest that our proposed robustness certificates are of high qualities and are computationally efficient even in large networks up to 7 layers or more than 10,000 neurons. In particular, we show that:
\begin{itemize}[nosep,wide,labelindent=0pt,labelwidth=*,align=left]

\item Our certified lower bounds (\fastlin, \fastlip) are close to (gap is only 2-3X) the exact minimum distortion computed by \reluplex for small networks (\reluplex is only feasible for networks with less 100 neurons for MNIST), but our algorithm is more than 10,000 times faster than \reluplex. See Table~\ref{tb:smallnetwork_main} and Table~\ref{tb:smallnetwork}.

\item Our certified lower bounds (\fastlin, \fastlip) give similar quality (the gap is within 35\%, and usually around 10\%; sometimes our bounds are even better) compared with the LP-based methods (\lp, \lpfull); however, our algorithm is 33 - 14,000 times faster. The LP-based methods are infeasible for networks with more than 4,000 neurons. See Table~\ref{tb:largenetwork_main} and Table \ref{tb:largenetwork_app}. 

\item When the network goes larger and deeper, our proposed methods can still give non-trivial lower bounds comparing to the upper bounds founded by attack algorithms on large networks. See Table~\ref{tb:largenetwork_main} and Table \ref{tb:largenetwork_app}.

\item For defended networks, especially for adversarial training~\cite{madry2017towards}, our methods give significantly larger bounds, validating the effectiveness of this defending method. Our algorithms can thus be used for evaluating defending techniques. See Table~\ref{tb:distill}.
\end{itemize}

\ifdef{\arxivversion}{

\begin{table*}[htbp]
\centering
\caption{Comparison of our proposed certified lower bounds \fastlin and \fastlip, \lp and \lpfull, the estimated lower bounds by \clever, the exact minimum distortion by \reluplex, and the upper bounds by \textbf{Attack} algorithms (CW $\ell_\infty$ for $p = \infty$, CW $\ell_2$ for $p = 2$, and EAD for $p = 1$) on 2, 3 layers toy MNIST networks with \textit{only 20 neurons per layer}. Differences of lower bounds and speedup are measured on the two corresponding \textbf{bold} numbers in each row, representing the best answer from our proposed algorithms and LP based approaches. \reluplex is designed to verify $\ell_\infty$ robustness so we omit results for $\ell_2$ and $\ell_1$. Note that \lpfull and \reluplex \textit{are very slow} and cannot scale to any practical networks, and the purpose of this table is to show how close our fast bounds are compared to the true minimum distortion provided by \textbf{Reluplex} and the bounds that are slightly tighter but very expensive (e.g. \lpfull).}
\label{tb:smallnetwork}

\begin{subtable}[t]{\textwidth}

\raggedright
\scalebox{0.76}{
\begin{tabular}{|c|c|c|
>{\centering\arraybackslash}p{9ex}
>{\centering\arraybackslash}p{9ex}|
>{\centering\arraybackslash}p{9ex}
>{\centering\arraybackslash}p{9ex}
||c||c||cc|}
\hline
\multicolumn{3}{|c|}{Toy Networks}                                                                                                   & \multicolumn{8}{c|}{Average Magnitude of Distortions on 100 Images}                                                                        \\ \hline
\multirow{3}{*}{Network}                                                       & \multirow{3}{*}{$p$}   & \multirow{3}{*}{Target}  & \multicolumn{4}{c||}{Certified Bounds}                                      & difference       & Exact                   & \multicolumn{2}{c|}{Uncertified}                    \\ \cline{4-7}\cline{9-11}
                                                                               &                             &                          & \multicolumn{2}{c|}{Our bounds} & \multicolumn{2}{c||}{Our baselines} & ours vs.        & Reluplex                & CLEVER                    &  Attacks                \\ \cline{4-7}
                                                                               &                             &                          & Fast-Lin     & Fast-Lip         & LP         & LP-Full                      & LP(-Full)         & \cite{katz2017reluplex} & \cite{weng2017evaluating} & CW/EAD                  \\ \hline

\multirow{9}{*}{\begin{tabular}[c]{@{}c@{}}MNIST\\ $2\times[20]$\end{tabular}} & \multirow{3}{*}{$\infty$} & runner-up & \bf 0.0191 & 0.0167     & \bf 0.0197 & 0.0197     &  -3.0\% & 0.04145  & 0.0235 & 0.04384 \\
                                                                               &                           & rand      & \bf 0.0309 & 0.0270     & \bf 0.0319 & 0.0319     &  -3.2\% & 0.07765  & 0.0428 & 0.08060 \\
                                                                               &                           & least     & \bf 0.0448 & 0.0398     & \bf 0.0462 & 0.0462     &  -3.1\% & 0.11711  & 0.0662 & 0.1224 \\ \cline{2-11}
                                                                               & \multirow{3}{*}{$2$}      & runner-up & \bf 0.3879 & 0.3677     & 0.4811     & \bf 0.5637 & -31.2\% & -        & 0.4615 & 0.64669 \\
                                                                               &                           & rand      & \bf 0.6278 & 0.6057     & 0.7560     & \bf 0.9182 & -31.6\% & -        & 0.8426 & 1.19630 \\
                                                                               &                           & least     & \bf 0.9105 & 0.8946     & 1.0997     & \bf 1.3421 & -32.2\% & -        & 1.315  & 1.88830 \\ \cline{2-11}
                                                                               & \multirow{3}{*}{$1$}      & runner-up & 2.3798     & \bf 2.8086 & 2.5932     & \bf 2.8171 &  -0.3\% & -        & 3.168  & 5.38380 \\
                                                                               &                           & rand      & 3.9297     & \bf 4.8561 & 4.2681     & \bf 4.6822 &  +3.7\% & -        & 5.858  & 11.4760 \\
                                                                               &                           & least     & 5.7298     & \bf 7.3879 & 6.2062     & \bf 6.8358 &  +8.1\% & -        & 9.250  & 19.5960 \\ \hline
\multirow{9}{*}{\begin{tabular}[c]{@{}c@{}}MNIST\\ $3\times[20]$\end{tabular}} & \multirow{3}{*}{$\infty$} & runner-up & \bf 0.0158 & 0.0094     & 0.0168     & \bf 0.0171 &  -7.2\% & 0.04234  & 0.0223 & 0.04786 \\
                                                                               &                           & rand      & \bf 0.0229 & 0.0142     & 0.0241     & \bf 0.0246 &  -6.9\% & 0.06824  & 0.0385 & 0.08114 \\
                                                                               &                           & least     & \bf 0.0304 & 0.0196     & 0.0319     & \bf 0.0326 &  -6.9\% & 0.10449  & 0.0566 & 0.11213 \\ \cline{2-11}
                                                                               & \multirow{3}{*}{$2$}      & runner-up & \bf 0.3228 & 0.2142     & 0.3809     & \bf 0.4901 & -34.1\% & -        & 0.4231 & 0.74117 \\
                                                                               &                           & rand      & \bf 0.4652 & 0.3273     & 0.5345     & \bf 0.7096 & -34.4\% & -        & 0.7331 & 1.22570 \\
                                                                               &                           & least     & \bf 0.6179 & 0.4454     & 0.7083     & \bf 0.9424 & -34.4\% & -        & 1.100  & 1.71090 \\ \cline{2-11}
                                                                               & \multirow{3}{*}{$1$}      & runner-up & \bf 2.0189 & 1.8819     & 2.2127     & \bf 2.5010 & -19.3\% & -        & 2.950  & 6.13750 \\
                                                                               &                           & rand      & \bf 2.8550 & 2.8144     & 3.1000     & \bf 3.5740 & -20.1\% & -        & 4.990  & 10.7220 \\
                                                                               &                           & least     & 3.7504     & \bf 3.8043 & 4.0434     & \bf 4.6967 & -19.0\% & -        & 7.131  & 15.6850 \\ \hline
\end{tabular}
}
\caption{Comparison of bounds}
\end{subtable}
\begin{subtable}[t]{\textwidth}
\centering
\raggedright
\scalebox{0.76}{
\begin{tabular}{|c|c|c|
>{\centering\arraybackslash}p{9ex}
>{\centering\arraybackslash}p{9ex}|
>{\centering\arraybackslash}p{9ex}
>{\centering\arraybackslash}p{9ex}
||c||c|}
\hline
\multicolumn{3}{|c|}{Toy Networks}                                                                                                     & \multicolumn{6}{c|}{Average Running Time per Image}                                                             \\ \hline
\multirow{3}{*}{Network}                                                       & \multirow{2}{*}{$p$}        & \multirow{3}{*}{Target} & \multicolumn{4}{c||}{Certified Bounds}                                      &       Exact     & Speedup         \\ \cline{4-8}
                                                                               &                             &                         & \multicolumn{2}{c|}{Our bounds} & \multicolumn{2}{c||}{Our baselines} &    Reluplex     & ours vs.       \\ \cline{4-7}

                                                                               &                             &           & Fast-Lin    & Fast-Lip    & LP         & LP-Full    & \cite{katz2017reluplex} &    LP-(full)     \\ \hline
\multirow{9}{*}{\begin{tabular}[c]{@{}c@{}}MNIST\\ $2\times[20]$\end{tabular}} & \multirow{3}{*}{$\infty$}   & runner-up & \bf 3.09 ms & 3.49 ms     & \bf 217 ms & 1.74 s     & 134 s    & 70X    \\
                                                                               &                             & rand      & \bf 3.25 ms & 5.53 ms     & \bf 234 ms & 1.93 s     & 38 s     & 72X    \\
                                                                               &                             & least     & \bf 3.37 ms & 8.90 ms     & \bf 250 ms & 1.97 s     & 360 s    & 74X    \\ \cline{2-9}
                                                                               & \multirow{3}{*}{$2$}        & runner-up & \bf 3.00 ms & 3.76 ms     & 1.10 s     & \bf 20.6 s & -        & 6864X  \\
                                                                               &                             & rand      & \bf 3.37 ms & 6.16 ms     & 1.20 s     & \bf 23.1 s & -        & 6838X  \\
                                                                               &                             & least     & \bf 3.29 ms & 9.89 ms     & 1.27 s     & \bf 26.4 s & -        & 8021X  \\ \cline{2-9}
                                                                               & \multirow{3}{*}{$1$}        & runner-up & 2.85 ms     & \bf 39.2 ms & 1.27 s     & \bf 16.1 s & -        & 412X   \\
                                                                               &                             & rand      & 3.32 ms     & \bf 54.8 ms & 1.59 s     & \bf 17.3 s & -        & 316X   \\
                                                                               &                             & least     & 3.46 ms     & \bf 68.1 ms & 1.74 s     & \bf 17.7 s & -        & 260X   \\ \hline
\multirow{9}{*}{\begin{tabular}[c]{@{}c@{}}MNIST\\ $3\times[20]$\end{tabular}} & \multirow{3}{*}{$\infty$}   & runner-up & \bf 5.58 ms & 3.64 ms     & 253 ms     & \bf 6.12 s & 4.7 hrs  & 1096X  \\
                                                                               &                             & rand      & \bf 6.12 ms & 5.23 ms     & 291 ms     & \bf 7.16 s & 11.6 hrs & 1171X  \\
                                                                               &                             & least     & \bf 6.62 ms & 7.06 ms     & 307 ms     & \bf 7.30 s & 12.6 hrs & 1102X  \\ \cline{2-9}
                                                                               & \multirow{3}{*}{$2$}        & runner-up & \bf 5.35 ms & 3.95 ms     & 1.22 s     & \bf 57.5 s & -        & 10742X \\
                                                                               &                             & rand      & \bf 5.86 ms & 5.81 ms     & 1.27 s     & \bf 66.3 s & -        & 11325X \\
                                                                               &                             & least     & \bf 5.94 ms & 7.55 ms     & 1.34 s     & \bf 77.3 s & -        & 13016X \\ \cline{2-9}
                                                                               & \multirow{3}{*}{$1$}        & runner-up & \bf 5.45 ms & 39.6 ms     & 1.27 s     & \bf 75.0 s & -        & 13763X \\
                                                                               &                             & rand      & \bf 5.56 ms & 52.9 ms     & 1.47 s     & \bf 82.0 s & -        & 14742X \\
                                                                               &                             & least     & 6.07 ms     & \bf 65.9 ms & 1.68 s     & \bf 85.9 s & -        & 1304X  \\ \hline
\end{tabular}
}
\caption{Comparison of time}
\end{subtable}
\end{table*}

\begin{table*}[htbp]
\centering
\caption{Comparison of our proposed certified lower bounds \fastlin and \fastlip with other lower bounds (\lp, \opnorm, \clever) and upper bounds (\textbf{Attack} algorithms: CW for $p = 2, \infty$, EAD for $p = 1$) on networks with 2-7 layers, where each layer has 1024 or 2048 nodes. Differences of lower bounds and speedup are measured on the two corresponding \textbf{bold} numbers in each row. Note that \lpfull and \reluplex are \textit{computationally infeasible} for all the networks reported here, and ``\textbf{-}'' indicates the method is computationally infeasible for that network. For \opnorm, computation time for each image is negligible as the operator norms can be pre-computed.}
\ifdef{\arxivversion}{
\label{tb:largenetwork}
}{
\label{tb:largenetwork_app}
}
\begin{adjustbox}{max width=\textwidth}
\begin{tabular}{|c|c|c|cc|cc||c||cc||cc|c||c|}
\hline

\multicolumn{3}{|c|}{Large Networks}                                                                                                   & \multicolumn{7}{c||}{Average Magnitude of Distortion on 100 Images}                                                                                                             & \multicolumn{4}{c|}{Average Running Time per Image}       \\ \hline
\multirow{3}{*}{Network}                                                      & \multirow{3}{*}{$p$}        & \multirow{3}{*}{Target}  & \multicolumn{4}{c||}{Certified Bounds}                                                     &  diff            & \multicolumn{2}{c||}{Uncertified}               & \multicolumn{3}{c||}{Certified Bounds}                   & Speedup        \\ \cline{4-7}\cline{9-13}
                                                                              &                             &                          & \multicolumn{2}{c|}{Our bounds} &     LP                  &  Op-norm                       &  ours          & CLEVER                      & Attacks           & \multicolumn{2}{c|}{Our bounds}  &    LP                 & ours         \\ \cline{4-5}\cline{11-12}
                                                                              &                             &                          &     Fast-Lin &   Fast-Lip       &   (Baseline)     &  \cite{szegedy2013intriguing}  &  vs. LP            & \cite{weng2017evaluating}   & CW/EAD            &     Fast-Lin      &     Fast-Lip & (Baseline)   & vs. LP           \\ \hline

\multirow{9}{*}{\begin{tabular}[c]{@{}c@{}}MNIST\\  $2 \times [1024]$\end{tabular}} & \multirow{3}{*}{$\infty$} & runner-up                   & \bf 0.02256 & 0.01802     & \bf 0.02493 & 0.00159 & -9.5\%  & 0.0447 & 0.0856 & \bf 127 ms & 167 ms     & \bf 19.3 s & 151X  \\
                                                                                    &                           & rand                        & \bf 0.03083 & 0.02512     & \bf 0.03386 & 0.00263 & -8.9\%  & 0.0708 & 0.1291 & \bf 156 ms & 219 ms     & \bf 20.8 s & 133X  \\
                                                                                    &                           & least                       & \bf 0.03854 & 0.03128     & \bf 0.04281 & 0.00369 & -10.0\% & 0.0925 & 0.1731
                                                                                    & \bf 129 ms & 377 ms     & \bf 22.2 s & 172X  \\ \cline{2-14}
                                                                                    & \multirow{3}{*}{$2$}      & runner-up                   & \bf 0.46034 & 0.42027     & \bf 0.55591 & 0.24327 & -17.2\% & 0.8104 & 1.1874 & \bf 127 ms & 196 ms     & \bf 419 s  & 3305X \\
                                                                                    &                           & rand                        & \bf 0.63299 & 0.59033     & \bf 0.75164 & 0.40201 & -15.8\% & 1.2841 & 1.8779 & \bf 128 ms & 234 ms     & \bf 195 s  & 1523X \\
                                                                                    &                           & least                       & \bf 0.79263 & 0.73133     & \bf 0.94774 & 0.56509 & -16.4\% & 1.6716 & 2.4556 & \bf 163 ms & 305 ms     & \bf 156 s  & 956X  \\ \cline{2-14}
                                                                                    & \multirow{3}{*}{$1$}      & runner-up                   & 2.78786     & \bf 3.46500 & \bf 3.21866 & 0.20601 & +7.7\%  & 4.5970 & 9.5295 & 117 ms     & \bf 1.17 s & \bf 38.9 s & 33X   \\
                                                                                    &                           & rand                        & 3.88241     & \bf 5.10000 & \bf 4.47158 & 0.35957 & +14.1\% & 7.4186 & 17.259 & 139 ms     & \bf 1.40 s & \bf 48.1 s & 34X   \\
                                                                                    &                           & least                       & 4.90809     & \bf 6.36600 & \bf 5.74140 & 0.48774 & +10.9\% & 9.9847 & 23.933 & 151 ms     & \bf 1.62 s & \bf 53.1 s & 33X   \\ \hline
\multirow{9}{*}{\begin{tabular}[c]{@{}c@{}}MNIST\\  $3 \times [1024]$\end{tabular}} & \multirow{3}{*}{$\infty$} & runner-up                   & \bf 0.01830 & 0.01021     & \bf 0.02013 & 0.00004 & -9.1\%  & 0.0509 & 0.1037 & \bf 1.20 s & 1.81 s     & \bf 50.4 s & 42X   \\
                                                                                    &                           & rand                        & \bf 0.02216 & 0.01236     & \bf 0.02428 & 0.00007 & -8.7\%  & 0.0717 & 0.1484 & \bf 1.12 s & 1.11 s     & \bf 52.7 s & 47X   \\
                                                                                    &                           & least                       & \bf 0.02432 & 0.01384     & \bf 0.02665 & 0.00009 & -8.7\%  & 0.0825 & 0.1777 & \bf 1.02 s & 924 ms     & \bf 54.3 s & 53X   \\ \cline{2-14}
                                                                                    & \multirow{3}{*}{$2$}      & runner-up                   & \bf 0.35867 & 0.22120     & \bf 0.41040 & 0.06626 & -12.6\% & 0.8402 & 1.3513 & \bf 898 ms & 1.59 s     & \bf 438 s  & 487X  \\
                                                                                    &                           & rand                        & \bf 0.43892 & 0.26980     & \bf 0.49715 & 0.10233 & -11.7\% & 1.2441 & 2.0387 & \bf 906 ms & 914 ms     & \bf 714 s  & 788X  \\
                                                                                    &                           & least                       & \bf 0.48361 & 0.30147     & \bf 0.54689 & 0.13256 & -11.6\% & 1.4401 & 2.4916 & \bf 925 ms & 1.01 s     & \bf 858 s  & 928X  \\ \cline{2-14}
                                                                                    & \multirow{3}{*}{$1$}      & runner-up                   & \bf 2.08887 & 1.80150     & \bf 2.36642 & 0.00734 & -11.7\% & 4.8370 & 10.159 & \bf 836 ms & 3.16 s     & \bf 91.1 s & 109X  \\
                                                                                    &                           & rand                        & \bf 2.59898 & 2.25950     & \bf 2.91766 & 0.01133 & -10.9\% & 7.2177 & 17.796 & \bf 863 ms & 3.84 s     & \bf 109 s  & 126X  \\
                                                                                    &                           & least                       & \bf 2.87560 & 2.50000     & \bf 3.22548 & 0.01499 & -10.8\% & 8.3523 & 22.395 & \bf 900 ms & 4.20 s     & \bf 122 s  & 136X  \\ \hline
\multirow{9}{*}{\begin{tabular}[c]{@{}c@{}}MNIST\\  $4 \times [1024]$\end{tabular}} & \multirow{3}{*}{$\infty$} & runner-up                   & \bf 0.00715 & 0.00219     & -           & 0.00001 & -       & 0.0485 & 0.08635 & \bf 1.90 s & 4.58 s     & -          & -     \\
                                                                                    &                           & rand                        & \bf 0.00823 & 0.00264     & -           & 0.00001 & -       & 0.0793 & 0.1303 & \bf 2.25 s & 3.08 s     & -          & -     \\
                                                                                    &                           & least                       & \bf 0.00899 & 0.00304     & -           & 0.00001 & -       & 0.1028 & 0.1680 & \bf 2.15 s & 3.02 s     & -          & -     \\ \cline{2-14}
                                                                                    & \multirow{3}{*}{$2$}      & runner-up                   & \bf 0.16338 & 0.05244     & -           & 0.11015 & -       & 0.8689 & 1.2422 & \bf 2.23 s & 3.50 s     & -          & -     \\
                                                                                    &                           & rand                        & \bf 0.18891 & 0.06487     & -           & 0.17734 & -       & 1.4231 & 1.8921 & \bf 2.37 s & 2.72 s     & -          & -     \\
                                                                                    &                           & least                       & \bf 0.20672 & 0.07440     & -           & 0.23710 & -       & 1.8864 & 2.4451 & \bf 2.56 s & 2.77 s     & -          & -     \\ \cline{2-14}
                                                                                    & \multirow{3}{*}{$1$}      & runner-up                   & \bf 1.33794 & 0.58480     & -           & 0.00114 & -       & 5.2685 & 10.079 & \bf 2.42 s & 2.71 s     & -          & -     \\
                                                                                    &                           & rand                        & \bf 1.57649 & 0.72800     & -           & 0.00183 & -       & 8.9764 & 17.200 & \bf 2.42 s & 2.91 s     & -          & -     \\
                                                                                    &                           & least                       & \bf 1.73874 & 0.82800     & -           & 0.00244 & -       & 11.867 & 23.910 & \bf 2.54 s & 3.54 s     & -          & -     \\ \hline
\multirow{9}{*}{\begin{tabular}[c]{@{}c@{}}CIFAR\\  $5 \times [2048]$\end{tabular}} & \multirow{3}{*}{$\infty$} & runner-up                   & \bf 0.00137 & 0.00020     & -           & 0.00000 & -       & 0.0062 & 0.00950 & \bf 24.2 s & 60.4 s     & -          & -     \\
                                                                                    &                           & rand                        & \bf 0.00170 & 0.00030     & -           & 0.00000 & -       & 0.0147 & 0.02351 & \bf 26.2 s & 78.1 s     & -          & -     \\
                                                                                    &                           & least                       & \bf 0.00188 & 0.00036     & -           & 0.00000 & -       & 0.0208 & 0.03416 & \bf 27.8 s & 79.0 s     & -          & -     \\ \cline{2-14}
                                                                                    & \multirow{3}{*}{$2$}      & runner-up                   & \bf 0.06122 & 0.00951     & -           & 0.00156 & -       & 0.2712 & 0.3778 & \bf 34.0 s & 60.7 s     & -          & -     \\
                                                                                    &                           & rand                        & \bf 0.07654 & 0.01417     & -           & 0.00333 & -       & 0.6399 & 0.9497 & \bf 36.8 s & 49.4 s     & -          & -     \\
                                                                                    &                           & least                       & \bf 0.08456 & 0.01778     & -           & 0.00489 & -       & 0.9169 & 1.4379 & \bf 37.4 s & 49.8 s     & -          & -     \\ \cline{2-14}
                                                                                    & \multirow{3}{*}{$1$}      & runner-up                   & \bf 0.93835 & 0.22632     & -           & 0.00000 & -       & 4.0755 & 7.6529 & \bf 36.5 s & 70.6 s     & -          & -     \\
                                                                                    &                           & rand                        & \bf 1.18928 & 0.31984     & -           & 0.00000 & -       & 9.7145 & 21.643 & \bf 37.5 s & 53.6 s     & -          & -     \\
                                                                                    &                           & least                       & \bf 1.31904 & 0.38887     & -           & 0.00001 & -       & 12.793 & 34.497 & \bf 38.3 s & 48.6 s     & -          & -     \\ \hline
\multirow{9}{*}{\begin{tabular}[c]{@{}c@{}}CIFAR\\  $6 \times [2048]$\end{tabular}} & \multirow{3}{*}{$\infty$} & runner-up                   & \bf 0.00075 & 0.00005     & -           & 0.00000 & -       & 0.0054 & 0.00770 & \bf 37.2 s & 106 s      & -          & -     \\
                                                                                    &                           & rand                        & \bf 0.00090 & 0.00007     & -           & 0.00000 & -       & 0.0131 & 0.01866 & \bf 37.0 s & 119 s      & -          & -     \\
                                                                                    &                           & least                       & \bf 0.00095 & 0.00008     & -           & 0.00000 & -       & 0.0199 & 0.02868 & \bf 37.2 s & 126 s      & -          & -     \\ \cline{2-14}
                                                                                    & \multirow{3}{*}{$2$}      & runner-up                   & \bf 0.03463 & 0.00228     & -           & 0.00476 & -       & 0.2394 & 0.2979 & \bf 56.1 s & 99.5 s     & -          & -     \\
                                                                                    &                           & rand                        & \bf 0.04129 & 0.00331     & -           & 0.01079 & -       & 0.5860 & 0.7635 & \bf 60.2 s & 95.6 s     & -          & -     \\
                                                                                    &                           & least                       & \bf 0.04387 & 0.00385     & -           & 0.01574 & -       & 0.8756 & 1.2111 & \bf 61.8 s & 88.6 s     & -          & -     \\ \cline{2-14}
                                                                                    & \multirow{3}{*}{$1$}      & runner-up                   & \bf 0.59638 & 0.05647     & -           & 0.00000 & -       & 3.3569 & 6.0112 & \bf 57.2 s & 108 s      & -          & -     \\
                                                                                    &                           & rand                        & \bf 0.72178 & 0.08212     & -           & 0.00000 & -       & 8.2507 & 17.160 & \bf 61.4 s & 88.2 s     & -          & -     \\
                                                                                    &                           & least                       & \bf 0.77179 & 0.09397     & -           & 0.00000 & -       & 12.603 & 28.958 & \bf 62.1 s & 65.1 s     & -          & -     \\ \hline
\multirow{9}{*}{\begin{tabular}[c]{@{}c@{}}CIFAR\\  $7 \times [1024]$\end{tabular}} & \multirow{3}{*}{$\infty$} & runner-up                   & \bf 0.00119 & 0.00006     & -           & 0.00000 & -       & 0.0062 & 0.0102 & \bf 10.5 s & 27.3 s     & -          & -     \\
                                                                                    &                           & rand                        & \bf 0.00134 & 0.00008     & -           & 0.00000 & -       & 0.0112 & 0.0218 & \bf 10.6 s & 29.2 s     & -          & -     \\
                                                                                    &                           & least                       & \bf 0.00141 & 0.00010     & -           & 0.00000 & -       & 0.0148 & 0.0333 & \bf 11.2 s & 30.9 s     & -          & -     \\ \cline{2-14}
                                                                                    & \multirow{3}{*}{$2$}      & runner-up                   & \bf 0.05279 & 0.00308     & -           & 0.00020 & -       & 0.2661 & 0.3943 & \bf 16.3 s & 28.2 s     & -          & -     \\
                                                                                    &                           & rand                        & \bf 0.05938 & 0.00407     & -           & 0.00029 & -       & 0.5145 & 0.9730 & \bf 16.9 s & 27.3 s     & -          & -     \\
                                                                                    &                           & least                       & \bf 0.06249 & 0.00474     & -           & 0.00038 & -       & 0.6253 & 1.3709 & \bf 17.4 s & 27.6 s     & -          & -     \\ \cline{2-14}
                                                                                    & \multirow{3}{*}{$1$}      & runner-up                   & \bf 0.76647 & 0.07028     & -           & 0.00000 & -       & 4.815  & 7.9987 & \bf 16.9 s & 27.8 s     & -          & -     \\
                                                                                    &                           & rand                        & \bf 0.86467 & 0.09239     & -           & 0.00000 & -       & 8.630  & 22.180 & \bf 17.6 s & 26.7 s     & -          & -     \\
                                                                                    &                           & least                       & \bf 0.91127 & 0.10639     & -           & 0.00000 & -       & 11.44  & 31.529 & \bf 17.5 s & 23.5 s     & -          & -     \\ \hline \hline
\multirow{3}{*}{\begin{tabular}[c]{@{}c@{}}MNIST\\  $3 \times [1024]$\end{tabular}} & $\infty$                  & \multirow{3}{*}{untargeted} & \bf 0.01808 & 0.01016     & \bf 0.01985 & 0.00004 & -8.9\%  & 0.0458 & 0.0993 & \bf 915 ms & 2.17 s     & \bf 227 s  & 248X  \\
                                                                                    & $2$                       &                             & \bf 0.35429 & 0.21833     & -           & 0.06541 & -       & 0.7413 & 1.1118 & \bf 950 ms & 2.02 s     & -          & -     \\
                                                                                    & $1$                       &                             & \bf 2.05645 & 1.78300     & \bf 2.32921 & 0.00679 & -11.7\% & 3.9661 & 9.0044 & \bf 829 ms & 4.41 s     & \bf 537 s  & 648X  \\ \hline
\multirow{3}{*}{\begin{tabular}[c]{@{}c@{}}CIFAR\\  $5 \times [2048]$\end{tabular}} & $\infty$                  & \multirow{3}{*}{untargeted} & \bf 0.00136 & 0.00020     & -           & 0.00000 & -       & 0.0056 & 0.00950 & \bf 24.1 s & 72.9 s     & -          & -     \\
                                                                                    & $2$                       &                             & \bf 0.06097 & 0.00932     & -           & 0.00053 & -       & 0.2426 & 0.3702 & \bf 34.2 s & 77.0 s     & -          & -     \\
                                                                                    & $1$                       &                             & \bf 0.93429 & 0.22535     & -           & 0.00000 & -       & 3.6704 & 7.3687 & \bf 35.6 s & 90.2 s     & -          & -     \\ \hline

 \end{tabular}
\end{adjustbox}
\end{table*}

\begin{table*}[h!]
\centering
\caption{Comparison of the lower bounds for $\ell_\infty$ distortion found by our  algorithms on models with defensive distillation (DD)~\cite{papernot2016distillation} with temperature = 100 and adversarial training~\cite{madry2017towards} with $\epsilon = 0.3$ for three targeted attack classes.}
\label{tb:distill}
\vspace{0.3em}
\ifdef{\useicmlformat}
{
}{
}
\resizebox{\linewidth}{!}{
\begin{tabular}{|c|c|ccc|ccc|ccc|}
\cline{1-11}
                             & \multicolumn{1}{c|}{}       & \multicolumn{3}{c|}{runner-up target}        & \multicolumn{3}{c|}{random target}           & \multicolumn{3}{c|}{least-likely target}                     \\ \hline
\multicolumn{1}{|c|}{Network} & \multicolumn{1}{c|}{Method} & Undefended &    DD    & Adv. Training & Undefended &    DD    & Adv. Training & Undefended &    DD    & Adv. Training                \\ \hline

\multirow{2}{*}{\makecell{MNIST \\ 3*[1024]}} & Fast-Lin & 0.01826 & 0.02724 & \bf 0.14730 & 0.02211 & 0.03827 & \bf 0.17275 & 0.02427 & 0.04967 & \bf 0.20136 \\
                                              & Fast-Lip & 0.00965 & 0.01803 & 0.09687     & 0.01217 & 0.02493 & 0.11618     & 0.01377 & 0.03207 & 0.13858     \\ \hline
\multirow{2}{*}{\makecell{MNIST \\ 4*[1024]}} & Fast-Lin & 0.00715 & 0.01561 & \bf 0.09579 & 0.00822 & 0.02045 & \bf 0.11209 & 0.00898 & 0.02368 & \bf 0.12901 \\
                                              & Fast-Lip & 0.00087 & 0.00585 & 0.04133     & 0.00145 & 0.00777 & 0.05048     & 0.00183 & 0.00903 & 0.06015     \\ \hline

\end{tabular}}
\end{table*}

}

\vspace{-1em}
\section{Conclusions}
In this paper we have considered the problem of verifying the robustness property of ReLU networks. By exploiting the special properties of ReLU networks, we have here presented two computational efficient methods \fastlin and \fastlip for this problem. Our algorithms are two orders of magnitude (or more) faster than LP-based methods, while obtaining solutions with similar quality; meanwhile, our bounds qualities are much better than the previously proposed operator-norm based methods. Additionally, our methods are efficient and easy to implement: we compute the bounds layer-by-layer, and the computation cost for each layer is similar to the cost of matrix products in forward propagation; moreover, we do not need to solve any integer programming, linear programming problems or their duals. Future work could extend our algorithm to handle the structure of convolutional layers and apply our algorithm to evaluate the robustness property of large DNNs such as ResNet on the ImageNet dataset.


\newpage
\section*{Acknowledgment} 
The authors sincerely thank Aviad Rubinstein for the suggestion of using set-cover to prove hardness. The authors sincerely thank Dana Moshkovitz for pointing out some references about the hardness result of set-cover. The authors would also like to thank Mika G\"{o}\"{o}s, Rasmus Kyng, Zico Kolter, Jelani Nelson, Eric Price, Milan Rubinstein, Jacob Steinhardt, Zhengyu Wang, Eric Wong and David P. Woodruff for useful discussions. Luca Daniel and Tsui-Wei Weng acknowledge the partial support of MIT-Skoltech program and MIT-IBM Watson AI Lab. Huan Zhang and Cho-Jui Hsieh acknowledge the support of NSF via IIS-1719097 and the computing resources provided by Google Cloud and NVIDIA.

\bibliography{reference}
\ifdef{\useicmlformat}{
\bibliographystyle{icml2018}
}{
\bibliographystyle{alpha}
}


\onecolumn
\newpage
\appendix
\counterwithin{table}{section}
\section{Hardness}\label{app:hardness}In this section we show that finding the minimum adversarial distortion with a certified approximation ratio is hard. We first introduce some basic definitions and theorems in Section~\ref{sec:hardness_definition}. We provide some backgrounds about in-approximability reduction in Section~\ref{sec:hardness_pcp}.  
Section~\ref{sec:hardness_warmup} gives a warmup proof for boolean case and then Section~\ref{sec:hardness_main_result} provides the proof of our main hardness result (for network with real inputs).

\subsection{Definitions}\label{sec:hardness_definition}
We provide some basic definitions and theorems in this section. First, we define the classic $\mathsf{3SAT}$ problem.
\begin{definition}[$\mathsf{3SAT}$ problem]\label{def:3sat}
Given $n$ variables and $m$ clauses in a conjunctive normal form $\mathsf{CNF}$ formula with the size of each clause at most $3$, the goal is to decide whether there exists an assignment to the $n$ Boolean variables to make the $\mathsf{CNF}$ formula to be satisfied.
\end{definition}

For the $\mathsf{3SAT}$ problem in Definition~\ref{def:3sat}, we introduce the Exponential Time Hypothesis (ETH),  which is a common concept in complexity field.
\begin{hypothesis}[Exponential Time Hypothesis ($\mathsf{ETH}$) \cite{ipz98}]
\label{hypo:eth}
There is a $\delta > 0$ such that the $\mathsf{3SAT}$ problem defined in Definition~\ref{def:3sat} cannot be solved in $O(2^{\delta n})$ time.
\end{hypothesis}

ETH had been used in many different problems, e.g. clustering \cite{abjk18,cmrr18}, low-rank approximation \cite{rsw16,swz17,swz17b,swz18}. For more details, we refer the readers to a survey \cite{lms13}.


Then we define another classical question in complexity theory, the $\mathsf{SET}$-$\mathsf{COVER}$ problem, which we will use in our proof. The exact $\mathsf{SET}$-$\mathsf{COVER}$ problem is one of Karp's 21 NP-complete problems known to be NP-complete in 1972: 

\begin{definition}[$\mathsf{SET}$-$\mathsf{COVER}$]
\label{def:set-cover}
The inputs are $U,S$; $U = \{ 1,2, \cdots, n \}$ is a universe, $P(U)$ is the power set of $U$, and $S = \{ S_1, \cdots, S_m \} \subseteq P(U)$ is a family of subsets, $\cup_{j \in [m]} S_j = U$. The goal is to give a {\rm YES/NO} answer to the follow decision problem: \\
\centerline{
Does there exist a set-cover of size $t$, i.e., $\exists C \subseteq [m]$, such that $\cup_{j \in C} S_j = U$ with $|C| = t$?}
\end{definition}

Alternatively, we can also state the problem as finding the minimum set cover size $t_0$, via a binary search on $t$ using the answers of the decision problem in~\ref{def:set-cover}. The Approximate $\mathsf{SET}$-$\mathsf{COVER}$ problem is defined as follows.

\begin{definition}[Approximate $\mathsf{SET}$-$\mathsf{COVER}$]
The inputs are $U,S$; $U = \{ 1,2, \cdots, n \}$ is a universe, $P(U)$ is the power set of $U$, and $S = \{ S_1, \cdots, S_m \} \subseteq P(U)$ is a family of subsets, $\cup_{j \in [m]} S_j = U$. The goal is to distinguish between the following two cases: \\
\rm{(\RN{1})}: There exists a small set-cover, i.e., $\exists C \subseteq [m]$, such that $\cup_{j \in C} S_j = U$ with $|C|\leq t$.\\
\rm{(\RN{2})}: Every set-cover is large, i.e., every $C \subseteq [m]$ with $\cup_{j \in C} S_j = U$ satisfies that $|C| > \alpha t$, where $\alpha >1$.
\end{definition}

An oracle that solves the Approximate $\mathsf{SET}$-$\mathsf{COVER}$ problem outputs an answer $t_U \geq t_0$ but $t_U \leq \alpha t_0$ using a binary search, where $t_U$ is an upper bound of $t_0$ with a guaranteed approximation ratio $\alpha$. For example, we can use a greedy (rather than exact) algorithm to solve the $\mathsf{SET}$-$\mathsf{COVER}$ problem, which cannot always find the smallest size of set cover $t_0$, but the size $t_U$ given by the greedy algorithm is at most $\alpha$ times as large as $t_0$.

In our setting, we want to investigate the hardness of finding the lower bound with a guaranteed approximation ration, but an approximate algorithm for $\mathsf{SET}$-$\mathsf{COVER}$ gives us an upper bound of $t_0$ instead of an lower bound of $t_0$. However, in the following proposition, we show that finding an lower bound with an approximation ratio of $\alpha$ is as hard as finding an upper bound with an approximation ratio of $\alpha$.

\begin{proposition}
\label{prop:lower-higher}
Finding a lower bound $t_L$ for the size of the minimal set-cover (that has size $t_0$) with an approximation ratio $\alpha$ is as hard as finding an upper bound $t_U$ with an approximation ratio $\alpha$.
\end{proposition}
\begin{proof}
If we find a lower bound $t_L$ with $\frac{t_0}{\alpha} \leq t_L \leq t_0 $, by multiplying both sides by $\alpha$, we also find an upper bound $t_U = \alpha t_L$ which satisfies that $t_0 \leq t_U \leq \alpha t_0$. So finding an lower bound with an approximation ratio $\alpha$ is at least as hard as finding an upper bound with an approximation ratio $\alpha$. The converse is also true.
\end{proof}
$\mathsf{SET}$-$\mathsf{COVER}$ is a well-studied problem in the literature. Here we introduce a theorem from ~\cite{rs97,ams06,ds14} which implies the hardness of approximating $\mathsf{SET}$-$\mathsf{COVER}$.
\begin{theorem}[\cite{rs97,ams06,ds14}]\label{thm:approx_set_cover}
Unless $\mathsf{NP}=\mathsf{P}$, there is no polynomial time algorithm that gives a $(1-o(1))\ln n$-approximation to $\mathsf{SET}$-$\mathsf{COVER}$ problem with universe size $n$.
\end{theorem}

We now formally define our neural network robustness verification problems.

\begin{definition}[$\mathsf{ROBUST}$-$\mathsf{NET}$($\R$)]\label{def:robust_net_real}
Given an $n$ hidden nodes {\rm ReLU} neural network $F(x) : \R^d \rightarrow \R$ where all weights are fixed, for a query input vector $x \in \R^d$ with $F(x) \leq 0$. The goal is to give a {\rm YES/NO} answer to the following decision problem:\\ 
\centerline{
Does there exist a $y$ with $\| x - y \|_1 \leq r$ such that $F(y) > 0$?}
\end{definition}
With an oracle of the decision problem available, we can figure out the smallest $r$ (defined as $r_0$) such that there exists a vector $y$ with $\| x - y \|_1 \leq r$ and $F(y) > 0$ via a binary search.

We also define a binary variant of the $\mathsf{ROBUST}$-$\mathsf{NET}$ problem, denoted as $\mathsf{ROBUST}$-$\mathsf{NET}$($\mathbb{B}$). The proof for this variant is more straightforward than the real case, and will help the reader understand the proof for the real case.

\begin{definition}[$\mathsf{ROBUST}$-$\mathsf{NET}$($\mathbb{B}$)]
Given an $n$ hidden nodes {\rm ReLU} neural network $F(x) : \{0,1\}^d \rightarrow \{0,1\}$ where weights are all fixed, for a query input vector $x \in \{0,1\}^d$ with $F(x) = 0$. The goal is to give a {\rm YES/NO} answer to the following decision problem:\\ 
\centerline{
Does there exist a $y$ with $\| x - y \|_1 \leq r$ such that $F(y) = 1$?}
\end{definition}

Then, we define the approximate version of our neural network robustness verification problems.
\begin{definition}[Approximate $\mathsf{ROBUST}$-$\mathsf{NET}$(${\mathbb{B}}$)]\label{def:approx_net_bool}
Given an $n$ hidden nodes {\rm ReLU} neural network $F(x) : \{0,1\}^d \rightarrow \{0,1\}$ where weights are all fixed, for a query input vector $x \in \{0,1\}^d$ with $F(x) = 0$. The goal is to distinguish the following two cases :\\
\rm{(\RN{1}):} There exists a point $y$ such that $\| x - y \|_1 \leq r$ and $F(y)=1$.\\ 
\rm{(\RN{2}):} For all $y$ satisfies $\| x - y \|_1 \leq \alpha r$, the $F(y) = 0$, where $\alpha >1$.
\end{definition}

\begin{definition}[Approximate $\mathsf{ROBUST}$-$\mathsf{NET}$($\R$)]\label{def:approx_net_real}
Given an $n$ hidden nodes {\rm ReLU} neural network $F(x) : \R^d \rightarrow \R$ where weights are all fixed, for a query input vector $x \in \R^d$ with $F(x) \leq 0$. The goal is to distinguish the following two cases :\\
\rm{(\RN{1}):} There exists a point $y$ such that $\| x - y \|_1 \leq r$ and $F(y) > 0$.\\ 
\rm{(\RN{2}):} For all $y$ satisfies $\| x - y \|_1 \leq \alpha r$, the $F(y) \leq 0$, where $\alpha >1$.
\end{definition}

As an analogy to $\mathsf{SET}$-$\mathsf{COVER}$, an oracle that solves the Approximate $\mathsf{ROBUST}$-$\mathsf{NET}$($\R$) problem can output an answer $r \geq r_0$ but $r \leq \alpha r_0$, which is an upper bound of $r_0$ with a guaranteed approximation ratio $\alpha$. With a similar statement as in Proposition~\ref{prop:lower-higher}, if we divide the answer $r$ by $\alpha$, then we get a lower bound $r^\prime = \frac{r}{\alpha}$ where $r^\prime \geq \frac{r_0}{\alpha}$, which is a lower bound with a guaranteed approximation ratio. If we can solve Approximate $\mathsf{ROBUST}$-$\mathsf{NET}$($\R$), we can get a lower bound with a guaranteed approximation ratio, which is the desired goal of our paper.

\subsection{Background of the PCP theorem}\label{sec:hardness_pcp}
The famous Probabilistically Checkable Proofs (PCP) theorem is the cornerstone of the theory of computational hardness of approximation, which investigates the inherent difficulty in designing efficient approximation algorithms for various optimization problems.\footnote{\url{https://en.wikipedia.org/wiki/PCP_theorem}} The formal definition can be stated as follows,
\begin{theorem}[\cite{as98,almss98}]\label{thm:pcp}
Given a $\mathsf{SAT}$ formula $\phi$ of size $n$ we can in time polynomial in $n$ construct a set of $M$ tests satisfying the following:\\
\rm{(\RN{1})} : Each test queries a constant number $d$ of bits from a proof, and based on the outcome of the queries it either acceptes or reject $\phi$.\\
\rm{(\RN{2})} : (Yes Case / Completeness) If $\phi$ is satisfiable, then there exists a proof so that all tests accept $\phi$.\\
\rm{(\RN{3})} : (No Case / Soundness) If $\phi$ is not satifiable, then no proof will cause more than $M/2$ tests to accept $\phi$.
\end{theorem}
Note that PCP kind of reduction is slightly different from NP reduction, for more examples (e.g. maximum edge biclique, sparsest cut) about how to use PCP theorem to prove inapproximibility results, we refer the readers to \cite{ams11}.

\subsection{Warm-up}\label{sec:hardness_warmup}
We state our hardness result for $\mathsf{ROBUST}$-$\mathsf{NET}$(${\mathbb{B}}$) (boolean inputs case) in this section. The reduction procedure for network with boolean inputs is more straightforward and easier to understand than the real inputs case. 
\begin{theorem}
Unless $\mathsf{NP}=\mathsf{P}$, there is no polynomial time algorithm to give a $(1-o(1))\ln n$-approximation to $\mathsf{ROBUST}$-$\mathsf{NET}$$(\mathbb{B})$ problem (Definition~\ref{def:approx_net_bool}) with $n$ hidden nodes.
\end{theorem}
\begin{proof}
Consider a set-cover instance, let $S$ denote a set of sets $\{ S_1, S_2, \cdots, S_d\}$ where $s_j \subseteq [n],\forall j \in [d]$.

For each set $S_j$ we create an input node $u_j$. For each element $i\in [n]$, we create a hidden node $v_i$. For each $i \in [n]$ and $j \in [d]$, if $i \in S_j$, then we connect $u_j$ and $v_i$. We also create an output node $w$, for each $i\in [n]$, we connect node $v_i$ and node $w$.

Let ${\bf 1}_{i \in S_j}$ denote the indicator function that it is $1$ if $i \in S_j$ and $0$ otherwise. Let $T_i$ denote the set that $T_i = \{ j ~|~ i \in S_j, \forall j \in [d] \}$.  For each $i\in [n]$, we define an activation function $\phi_i$ satisfies that
\begin{align*}
\phi_i =
\begin{cases}
1, & \text{~if~} \sum_{j\in T_i} u_j \geq 1,\\
0, & \text{~otherwise}.
\end{cases}
\end{align*}
Since $u_j\in\{0,1\}$, $\phi_i$ can be implemented in this way using ReLU activations:
\begin{align*}
\phi_i = 1 - \max \left( 0, 1 - \sum_{j\in T_i} u_j \right).
\end{align*}
Note that $\sum_{j=1}^d {\bf 1}_{i \in S_j} = \sum_{j=1}^d u_j$, because $u_j=1$ indicates choosing set $S_j$ and $u_j=0$ otherwise.

For final output node $w$, we define an activation function $\psi$ satisfies that
\begin{align*}
\psi = 
\begin{cases}
1, & \text{~if~} \sum_{i=1}^n v_i \geq n, \\
0, & \text{~otherwise}.
\end{cases}
\end{align*}
Since $v_i\in[n]$, $\psi$ can be implemented as
\begin{align*}
\psi =  \max \left( 0,  \sum_{i=1}^n v_i -n+1\right).
\end{align*}
We use vector $x$ to denote $\{0\}^d$ vector and it is to easy to see that $F(x) = 0$. Let $\alpha > 1$ denote a fixed parameter. Also, we have $F(y)>0$ if and only if $C=\{j|y_j=1\}$ is a set-cover.
According to our construction, we can have the following two claims,

\begin{claim}[Completeness]\label{cla:bool_completeness}
If there exists a set-cover $C \subseteq [d]$ with $\cup_{j\in C} S_j = [n]$ and $|C| \leq r$, then there exists a point $y \in \{0,1\}^d$ such that $\| x - y \|_1 \leq r$ and $F(y) > 0$.
\end{claim}

\begin{claim}[Soundness]\label{cla:bool_soundness}
If for every $C \subseteq [d]$ with $\cup_{j \in C}S_j = U$ satisfies that $|C| > \alpha \cdot t$, then for all $y\in \{0,1\}^d$ satisfies that $\| x - y\|_1 \leq \alpha r$, $F(y) \leq 0$ holds.
\end{claim}
Therefore, using Theorem~\ref{thm:pcp}, Theorem~\ref{thm:approx_set_cover}, Claim~\ref{cla:bool_completeness} and Claim~\ref{cla:bool_soundness} completes the proof.
\end{proof}

\subsection{Main result}\label{sec:hardness_main_result}
With the proof for $\mathsf{ROBUST}$-$\mathsf{NET}$(${\mathbb{B}}$) as a warm-up, we now prove our main hardness result for $\mathsf{ROBUST}$-$\mathsf{NET}$(${\mathbb{R}}$) in this section. 

\begin{theorem}\label{thm:robust_net_R}
Unless $\mathsf{NP}=\mathsf{P}$, there is no polynomial time algorithm to give an $(1-o(1))\ln n$-approximation to $\mathsf{ROBUST}$-$\mathsf{NET}(\R)$ problem (Definition~\ref{def:approx_net_real}) with $n$ hidden nodes.
\end{theorem}
\begin{proof}
Consider a set-cover instance, let $S$ denote a set of sets $\{ S_1, S_2, \cdots, S_d\}$ where $S_j \subseteq [n],\forall j \in [d]$. For each set $S_j$ we create an input node $u_j$. For each $j \in [d]$, we create a hidden node $t_j$ and connect $u_j$ and $t_j$.

For each element $i\in [n]$, we create a hidden node $v_i$. For each $i \in [n]$ and $j \in [d]$, if $i \in S_j$, then we connect $u_j$ and $v_i$. Finally, we create an output node $w$ and for each $i\in [n]$, we connect node $v_i$ and node $w$.

Let $\delta = 1/ d$. For each $j \in [n]$, we apply an activation function $\phi_{1,j}$ on $t_j$ such that
\begin{align*}
\phi_{1,j} = - \max (0, \delta - u_j ) + \max (0, u_j - 1 + \delta)
\end{align*}
It is easy to see that
\begin{align*}
t_j=\phi_{1,j} =
\begin{cases}
u_j - \delta & \text{~if~} u_j \in [0,\delta]  \\
u_j -(1-\delta) & \text{~if~} u_j \in [1-\delta, 1]\\
0 & \text{~otherwise~}.
\end{cases}
\end{align*}

Let $T_i$ denote the set that $T_i = \{ j ~|~ i \in S_j, \forall j \in [d] \}$. 
For each $i\in [n]$, we need an activation function $\phi_{2,i}$ on node $v_i$ which satisfies that
\begin{align*}
\phi_{2,i} \in
\begin{cases}
[-\delta,0], & \text{~if~} \forall j \in T_i, t_j \in [-\delta,0], \\
[0,\delta], & \text{~if~} \exists j \in T_i, t_j \in [0,\delta].
\end{cases}
\end{align*}
This can be implemented in the following way,
\begin{align*}
\phi_{2,i} = \max_{j \in T_i} t_j. 
\end{align*}
For the final output node $w$, we define it as $$w=\min_{i \in [n]} v_i.$$
We use vector $x$ to denote $\{0\}^d$ vector and it is to easy to see that $F(x) =-\delta< 0$. Let $\alpha > 1$ denote a fixed parameter.

According to our construction, we can have the following two claims.
\begin{claim}[Completeness]\label{cla:real_completeness}
If there exists a set-cover $C \subseteq [d]$ with $\cup_{j\in C} S_j = [n]$ and $|C| \leq r$, then there exists a point $y \in [0,1]^d$ such that $\| x - y \|_1 \leq r$ and $F(y) > 0$.
\end{claim}
\begin{proof}
Without loss of generality, we let the set cover to be $\{S_1,S_2,...,S_r\}$. Let $y_1=y_2=\cdots=y_r=1$ and $y_{r+1}=y_{r+2}=...=y_{d}=0.$ By the definition of $t_j$, we have $t_1=t_2=\cdots=t_r=\delta.$ Since $\{S_1,S_2,\cdots,S_r\}$ is a set-cover, we know that $v_i=\delta$ for all $i\in[n]$. Then $F(y)=w=\min_{i\in[n ]}v_i=\delta>0.$ Since we also have $\|y\|_1=r,$ the adversarial point is found.
\end{proof}

\begin{claim}[Soundness]\label{cla:real_soundness}
If for every $C \subseteq [d]$ with $\cup_{j \in C}S_j = U$ satisfies that $|C| > \alpha \cdot r$, then for all $y\in [0,1]^d$ satisfies that $\| x - y\|_1 \leq \alpha r (1-1/ d )$, $F(y) \leq 0$ holds.
\end{claim}
\begin{proof}
Proof by contradiction. We assume that there exists $y$ such that $F(y)>0$ and $\| y\|_1 \leq  \alpha r (1-1/d)$. Since $F(y)>0$, we have for all $i$, $v_i>0$. Thus there exists $j\in T_i$ such that $t_j>0$. Let $\pi : [n] \rightarrow Q$ denote a mapping ($Q\subseteq[d]$ will be decided later). This means that for each $i \in [n]$, there exists $j\in T_i$, such that $1-\delta<y_{j}\leq 1$, and we let $\pi(i)$ denote that $j$.

 We define set $Q \subseteq [d]$ as follows
\begin{align*}
Q = \{ j ~|~  \exists i \in [n], \text{~s.t.~} \pi(i)=j \in T_i \text{~and~} t_j > 0 \}.
\end{align*}
 Since $\sum_{j\in[d]}|y_j|=\| y\|_1\leq  \alpha r (1-1/d) $, we have
\begin{align*}
\sum_{j \in Q }|y_j|\leq\sum_{ j \in [d] } |y_j| \leq \alpha r (1 - 1/d),
\end{align*}
where the first step follows by $|Q| \leq d$.

Because for all $j \in Q$, $|y_j|>1-\delta=1-1/d$, we have 
\begin{align*}
|Q|\leq \frac{ \alpha r ( 1 - 1/d ) }{ ( 1- 1/d )}= \alpha\cdot r.
\end{align*}
So $\{ S_{j} \}_{j \in Q}$ is a set-cover with size less than or equal to $\alpha\cdot r$, which is a contradiction. 
\end{proof}
Therefore, using Theorem~\ref{thm:pcp}, Theorem~\ref{thm:approx_set_cover}, Claim~\ref{cla:real_completeness} and Claim~\ref{cla:real_soundness} completes the proof.
\end{proof}

By making a stronger assumption of $\mathsf{ETH}$, we can have the following stronger result which excludes all $2^{o(n^c)}$ time algorithms, where $c>0$ is some fixed constant:

\begin{corollary}
Assuming Exponential Time Hypothesis ($\mathsf{ETH}$, see Hypothesis~\ref{hypo:eth}), there is no $2^{o(n^c)}$ time algorithm that gives a $(1-o(1))\ln n$-approximation to $\mathsf{ROBUST}$-$\mathsf{NET}$ problem with $n$ hidden nodes, where $c>0$ is some fixed constant.
\end{corollary}
\begin{proof}
It follows by the construction in Theorem~\ref{thm:robust_net_R} and \cite{m12,moshkovitz2012projection}. 
\end{proof}

Note that in~\cite{m12}, an additional conjecture, Projection Games Conjecture ($\mathsf{PGC}$) is required for the proof, but the result was improved in \cite{moshkovitz2012projection} and $\mathsf{PGC}$ is not a requirement any more.
\newpage
\section{Proof of Theorem \ref{thm:cvx_bnd}}
\label{app:approach1_explicit_function}
For a $m$-layer ReLU network, assume we know all the pre-ReLU activation bounds $\lwbnd{(k)}$ and $\upbnd{(k)}$, $\forall k \in[ m-1]$ for a $m$-layer ReLU network and we want to compute the bounds of the the $j$ th output at $m$ th layer.

The $j$ th output can be written as 
\begin{align}
\label{proof:eq:1stfj}
	f_j(\x) &= \sum_{k=1}^{n_{m-1}} \W{(m)}_{j,k} [\phi_{m-1}(\x)]_k + \bias{(m)}_j, \\
    &= \sum_{k=1}^{n_{m-1}} \W{(m)}_{j,k} \sigma(\W{(m-1)}_{k,:} \phi_{m-2}(x)+\bias{(m-1)}_k) + \bias{(m)}_j, \\
    &= \sum_{k \in \setIpos{m-1}, \setIneg{m-1}, \setIuns{m-1}} \W{(m)}_{j,k} \sigma(\W{(m-1)}_{k,:} \phi_{m-2}(\x)+\bias{(m-1)}_k) + \bias{(m)}_j.    
\end{align}
For neurons belonging to category (i), i.e., $k \in \setIpos{m-1}$, 
\begin{equation*}
	 \sigma(\W{(m-1)}_{k,:} \phi_{m-2}(x)+\bias{(m-1)}_k) = \W{(m-1)}_{k,:} \phi_{m-2}(\x)+\bias{(m-1)}_k. 
\end{equation*}
For neurons belonging to category (ii), i.e., $k \in \setIneg{m-1}$,
\begin{equation*}
	 \sigma(\W{(m-1)}_{k,:} \phi_{m-2}(\x)+\bias{(m-1)}_k) = 0. 
\end{equation*}
Finally, for neurons belonging to Category (iii), i.e., $k \in \setIuns{m-1}$, we bound their outputs. If we adopt the linear upper and lower bounds in \eqref{eq:our_cvx_approx} and let $\bm{d}^{(m-1)}_k := \frac{\upbnd{(m-1)}_k}{\upbnd{(m-1)}_k-\lwbnd{(m-1)}_k}$, we have 
\begin{equation}
\label{proof:eqcvx}
	 \bm{d}^{(m-1)}_k  (\W{(m-1)}_{k,:}  \phi_{m-2}(\x) + \bias{(m-1)}_k) \leq \sigma(\W{(m-1)}_{k,:}  \phi_{m-2}(\x) + \bias{(m-1)}_k) \leq \bm{d}^{(m-1)}_k  (\W{(m-1)}_{k,:}  \phi_{m-2}(\x) + \bias{(m-1)}_k - \lwbnd{(m-1)}_k). 
\end{equation}

\subsection{Upper bound}
The goal of this section is to prove Lemma~\ref{lem:approach_1_upper_bound}.
\begin{lemma}[Upper bound with explicit function]\label{lem:approach_1_upper_bound}
Given an $m$-layer ReLU neural network function $f : \R^{n_0} \rightarrow \R^{n_m}$, parameters $p, \epsilon$, there exists two explicit functions $f^L : \R^{n_0} \rightarrow \R^{n_m}$ and $f^U :\R^{n_0} \rightarrow \R^{n_m}$ (see Definition~\ref{def:f_L_f_U}) such that $\forall j \in [n_m]$,
\begin{equation*}
 f_{j}(\x) \leq f_{j}^{U}(\x), \forall \x \in B_p(\x_0,\epsilon).
\end{equation*}
\end{lemma}
Notice that \eqref{proof:eqcvx} can be used to construct an upper bound and lower bound of $f_j(\x)$ by considering the signs of the weights $\W{(m)}_{j,k}$. Let $f_j^{U,m-1}(\x)$ be an upper bound of $f_j(\x)$; $f_j^{U,m-1}(\x)$ can be constructed by taking the right-hand-side (RHS) of \eqref{proof:eqcvx} if $\W{(m)}_{j,k} > 0$ and taking the left-hand-side (LHS) of \eqref{proof:eqcvx} if $\W{(m)}_{j,k} < 0$:
\begin{align}
 & \quad f_j^{U,m-1}(\x)  \notag \\
& = \sum_{k\in \setIpos{m-1}} \W{(m)}_{j,k} (\W{(m-1)}_{k,:} \phi_{m-2}(\x) + \bias{(m-1)}_k) \label{proof:eq:1stfju} \\
 & \quad +\sum_{k\in \setIuns{m-1}, \W{(m)}_{j,k}>0} \W{(m)}_{j,k} \bm{d}^{(m-1)}_k (\W{(m-1)}_{k,:} \phi_{m-2}(\x) + \bias{(m-1)}_k - \lwbnd{(m-1)}_k) \nonumber \\ 
 & \quad +\sum_{k\in \setIuns{m-1}, \W{(m)}_{j,k}<0} \W{(m)}_{j,k} \bm{d}^{(m-1)}_k (\W{(m-1)}_{k,:} \phi_{m-2}(\x) + \bias{(m-1)}_k) + \bias{(m)}_j \nonumber \\
\label{proof:eq:1stfju1}
& = \sum_{k = 1}^{n_{m-1}} \W{(m)}_{j,k} \bm{d}^{(m-1)}_k (\W{(m-1)}_{k,:} \phi_{m-2}(\x) + \bias{(m-1)}_k) - \sum_{k\in \setIuns{m-1}, \W{(m)}_{j,k}>0} \W{(m)}_{j,k} \bm{d}^{(m-1)}_k \lwbnd{(m-1)}_k + \bias{(m)}_j, \\
\label{proof:eq:1stfju2}
& = \sum_{k = 1}^{n_{m-1}} \W{(m)}_{j,k} \bm{d}^{(m-1)}_k \W{(m-1)}_{k,:} \phi_{m-2}(\x) \\
& \quad +\left(\sum_{k = 1}^{n_{m-1}} \W{(m)}_{j,k} \bm{d}^{(m-1)}_k \bias{(m-1)}_k - \sum_{k\in \setIuns{m-1}, \W{(m)}_{j,k}>0} \W{(m)}_{j,k} \bm{d}^{(m-1)}_k \lwbnd{(m-1)}_k + \bias{(m)}_j \right), \nonumber
\end{align} 
where we set $\bm{d}^{(m-1)}_k = 1$ for $k \in \setIpos{m-1}$ and set $\bm{d}^{(m-1)}_k = 0$ for $k \in \setIneg{m-1}$ from \eqref{proof:eq:1stfju} to \eqref{proof:eq:1stfju1} and collect the constant terms (independent of $\x$) in the parenthesis from \eqref{proof:eq:1stfju1} to \eqref{proof:eq:1stfju2}. 

If we let $\A{(m-1)} = \W{(m)}\DD{(m-1)}$, where $\DD{(m-1)}$ is a diagonal matrix with diagonals being $\bm{d}^{(m-1)}_k$, then we can rewrite $f_j^{U,m-1}(\x)$  into the following:
\begin{align}
\label{proof:eq:1stfju3}
f_j^{U,m-1}(\x) & = \sum_{k = 1}^{n_{m-1}} \A{(m-1)}_{j,k} \W{(m-1)}_{k,:} \phi_{m-2}(\x) + \left(\A{(m-1)}_{j,:} \bias{(m-1)} - \A{(m-1)}_{j,:} \upbias{(m-1)}_{:,j} + \bias{(m)}_j \right) \\
\label{proof:eq:1stfju4}
& = \sum_{k = 1}^{n_{m-1}} \A{(m-1)}_{j,k} (\sum_{r=1}^{n_{m-2}} \W{(m-1)}_{k,r} [\phi_{m-2}(\x)]_r ) + \left(\A{(m-1)}_{j,:} \bias{(m-1)} - \A{(m-1)}_{j,:} \upbias{(m-1)}_{:,j} + \bias{(m)}_j \right) \\
\label{proof:eq:1stfju5}
& = \sum_{r=1}^{n_{m-2}} \sum_{k = 1}^{n_{m-1}} \A{(m-1)}_{j,k}  \W{(m-1)}_{k,r} [\phi_{m-2}(\x)]_r  + \left(\A{(m-1)}_{j,:} \bias{(m-1)} - \A{(m-1)}_{j,:} \upbias{(m-1)}_{:,j} + \bias{(m)}_j \right) \\
\label{proof:eq:1stfju6}
& = \sum_{r=1}^{n_{m-2}} \tilde{\bm{W}}^{(m-1)}_{j,r} [\phi_{m-2}(\x)]_r  + \tilde{\bm{b}}^{(m-1)}_j.
\end{align}
From \eqref{proof:eq:1stfju2} to \eqref{proof:eq:1stfju3}, we rewrite the summation terms in the parenthesis into matrix-vector multiplications and for each $j \in [n_m]$ let 
\begin{equation*}
\upbias{(m-1)}_{k,j} =   
    \begin{cases}
    \lwbnd{(m-1)}_k & \text{if $k \in \setIuns{m-1}, \, \A{(m-1)}_{j,k} > 0$} \\
    0 & \text{otherwise} 
  \end{cases}
\end{equation*}
since $0 \leq \bm{d}^{(m-1)}_k \leq 1$,  $\W{(m)}_{j,k} > 0$ is equivalent to $\A{(m-1)}_{j,k} > 0$.

From \eqref{proof:eq:1stfju3} to \eqref{proof:eq:1stfju4}, we simply write out the inner product $\W{(m-1)}_{k,:} \phi_{m-2}(\x)$ into a summation form, and from \eqref{proof:eq:1stfju4} to \eqref{proof:eq:1stfju5}, we exchange the summation order of $k$ and $r$. From \eqref{proof:eq:1stfju5} to \eqref{proof:eq:1stfju6}, we let 
\begin{align}
\label{proof:eq:1stfju7}
\tilde{\bm{W}}^{(m-1)}_{j,r} &= \sum_{k = 1}^{n_{m-1}} \A{(m-1)}_{j,k}  \W{(m-1)}_{k,r} \\
\label{proof:eq:1stfju8}
\tilde{\bm{b}}^{(m-1)}_j &= \left(\A{(m-1)}_{j,:} \bias{(m-1)} - \A{(m-1)}_{j,:} \upbias{(m-1)}_{:,j} + \bias{(m)}_j \right)
\end{align}
and now we have \eqref{proof:eq:1stfju6} in the same form as \eqref{proof:eq:1stfj}.

Indeed, in \eqref{proof:eq:1stfj}, the running index is $k$ and we are looking at the $m$ th layer, with weights $\W{(m)}_{j,k}$, activation functions $\phi_{m-1}(\x)$ and bias term $\bias{(m)}_j$; in \eqref{proof:eq:1stfju6}, the running index is $r$ and we are looking at the $m-1$ th layer with \textit{equivalent} weights $\tilde{\bm{W}}^{(m-1)}_{j,r}$, activation functions $\phi_{m-2}(\x)$ and \textit{equivalent} bias $\tilde{\bm{b}}^{(m-1)}_j$. Thus, we can use the same technique from \eqref{proof:eq:1stfj} to \eqref{proof:eq:1stfju6} and obtain an upper bound on the $f_j^{U,m-1}(\x)$ and repeat this procedure until obtaining $f_j^{U,1}(\x)$, where
$$f_j(\x) \leq f_j^{U,m-1}(\x) \leq f_j^{U,m-2}(\x) \leq \ldots \leq f_j^{U,1}(\x).$$ 

Let the final upper bound $f_j^U(\x) = f_j^{U,1}(\x)$, and now we have 
\begin{equation*}
	f_{j}(\x) \leq f_{j}^{U}(\x), 
\end{equation*}
where $f^U_j(\x) = [f^{U}(\x)]_j$,
\begin{align*}
		f^{U}_j(\x) & = \A{(0)}_{j,:}\x+ \bias{(m)}_j+\sum_{k=1}^{m-1}\A{(k)}_{j,:}(\bias{(k)}-\upbias{(k)}_{:,j}) 
\end{align*}
and for $k = 1, \, \ldots, \, m-1, $
\begin{equation*}
	\A{(m-1)} = \W{(m)} \DD{(m-1)}, \, \A{(k-1)} = \A{(k)} \W{(k)} \DD{(k-1)},  
\end{equation*}
\begin{align*}
	\DD{(0)} & = \bm{I}_{n_0}  \nonumber  \\
    \DD{(k)}_{r,r} & =   
    \begin{cases}
    \frac{\upbnd{(k)}_r}{\upbnd{(k)}_r-\lwbnd{(k)}_r} & \text{if $r \in \setIuns{k}$} \\
    1 & \text{if $r \in \setIpos{k}$} \\
    0 & \text{if $r \in \setIneg{k}$}     
  \end{cases}
  \\
      \upbias{(k)}_{r,j} & =   
    \begin{cases}
    \lwbnd{(k)}_r & \text{if $r \in \setIuns{k}, \, \A{(k)}_{j,r} > 0$} \\
    0 & \text{otherwise} 
  \end{cases}
\end{align*}

\subsection{Lower bound }
The goal of this section is to prove Lemma~\ref{lem:approach_1_lower_bound}.
\begin{lemma}[Lower bound with explicit function]\label{lem:approach_1_lower_bound}
Given an $m$-layer ReLU neural network function $f : \R^{n_0} \rightarrow \R^{n_m}$, parameters $p, \epsilon$, there exists two explicit functions $f^L : \R^{n_0} \rightarrow \R^{n_m}$ and $f^U :\R^{n_0} \rightarrow \R^{n_m}$ (see Definition~\ref{def:f_L_f_U}) such that $\forall j \in [n_m]$,
\begin{equation*}
  f_{j}^{L}(\x) \leq f_{j}(\x), \forall \x \in B_p(\x_0,\epsilon).
\end{equation*}
\end{lemma}

Similar to deriving the upper bound of $f_j(\x)$, we consider the signs of the weights $\W{(m)}_{j,k}$ to derive the lower bound. Let $f_j^{L,m-1}(\x)$ be a lower bound of $f_j(\x)$; $f_j^{L,m-1}(\x)$ can be constructed by taking the right-hand-side (RHS) of \eqref{proof:eqcvx} if $\W{(m)}_{j,k} < 0$ and taking the left-hand-side (LHS) of \eqref{proof:eqcvx} if $\W{(m)}_{j,k} > 0$. Following the procedure in \eqref{proof:eq:1stfju} to \eqref{proof:eq:1stfju6} (except that now the additional bias term is from the set $k \in \setIuns{m-1}, \W{(m)}_{j,k} < 0$), the lower bound is similar to the upper bound we have derived but but replace $\upbias{(m-1)}$ by $\lwbias{(m-1)}$, where for each $j \in [n_m]$,
\begin{equation*}
\lwbias{(m-1)}_{k,j} =   
    \begin{cases}
    \lwbnd{(m-1)}_k & \text{if $k \in \setIuns{m-1}, \, \A{(m-1)}_{j,k} < 0$} \\
    0 & \text{otherwise.} 
  \end{cases}
\end{equation*}
It is because the linear upper and lower bounds in \eqref{eq:our_cvx_approx} has the same slope $\frac{u}{u-l}$ on both sides (i.e. $\sigma(y)$ is bounded by two lines with the same slope but different intercept), which gives the same $\A{}$ matrix and $\DD{}$ matrix in computing the upper bound and lower bound of $f_j(\x)$. This is the key to facilitate a faster computation under this linear approximation \eqref{eq:our_cvx_approx}. Thus, the lower bound for $f_j(\x)$ is:
\begin{equation*}
	f_{j}^{L}(\x) \leq f_{j}(\x) , 
\end{equation*}
where $f^L_j(\x) = [f^{L}(\x)]_j$,
\begin{align*}
        f^{L}_j(\x) & = \A{(0)}_{j,:}\x+ \bias{(m)}_j+\sum_{k=1}^{m-1}\A{(k)}_{j,:}(\bias{(k)}-\lwbias{(k)}_{:,j})
\end{align*}
and for $k = 1, \, \ldots, \, m-1, $
\begin{equation*}
      \lwbias{(k)}_{r,j}  =   
    \begin{cases}
    \lwbnd{(k)}_r & \text{if $r \in \setIuns{k}, \, \A{(k)}_{j,r} < 0$} \\
    0 & \text{otherwise.} 
  \end{cases}
\end{equation*}

\newpage
\section{Proof of Corollary \ref{cor:cvx_bnd}}\label{app:approach1_fixed_value}
By Theorem \ref{thm:cvx_bnd}, for $\x \in B_p(\xo,\epsilon)$, we have $ f_j^{L}(\x) \leq f_j(\x) \leq f_j^{U}(\x)$. Thus,  
\begin{align}
f_j(\x) &\leq f_j^{U}(\x) \leq \max_{\x \in B_p(\x,\epsilon)} f_j^{U}(\x), \\
f_j(\x) &\geq f_j^{L}(\x) \geq \min_{\x \in B_p(\x,\epsilon)} f_j^{L}(\x).
\end{align}
Since $f_j^{U}(\x) = \A{(0)}_{j,:}\x+ \bias{(m)}_j+\sum_{k=1}^{m-1}\A{(k)}_{j,:}(\bias{(k)}-\upbias{(k)}_{:,j})$, 
\begin{align}
\gamma^U_j := 
\max_{\x \in B_p(\xo,\epsilon)} f_j^{U}(\x) 
&= \max_{\x \in B_p(\xo,\epsilon)} \left( \A{(0)}_{j,:}\x+ \bias{(m)}_j+\sum_{k=1}^{m-1}\A{(k)}_{j,:}(\bias{(k)}-\upbias{(k)}_{:,j}) \right) \nonumber \\
&= \left( \max_{\x \in B_p(\xo,\epsilon)} \A{(0)}_{j,:}\x \right) + \bias{(m)}_j+\sum_{k=1}^{m-1}\A{(k)}_{j,:}(\bias{(k)}-\upbias{(k)}_{:,j}) \label{proof:cor:max1} \\
&= \epsilon \left( \max_{\bm{y} \in B_p(\bm{0},1)} \A{(0)}_{j,:}\bm{y} \right) + \A{(0)}_{j,:}\xo + \bias{(m)}_j+\sum_{k=1}^{m-1}\A{(k)}_{j,:}(\bias{(k)}-\upbias{(k)}_{:,j}) \label{proof:cor:max2} \\
&= \epsilon \| \A{(0)}_{j,:}\|_q + \A{(0)}_{j,:}\xo + \bias{(m)}_j+\sum_{k=1}^{m-1}\A{(k)}_{j,:}(\bias{(k)}-\upbias{(k)}_{:,j}). \label{proof:cor:max3}
\end{align}
From \eqref{proof:cor:max1} to \eqref{proof:cor:max2}, we do a transformation of variable $\bm{y} := \frac{\x-\xo}{\epsilon}$ and therefore $\bm{y} \in B_p(\bm{0},1)$. By the definition of dual norm $\| \cdot \|_*$:
\begin{equation*}
	\| \bm{z} \|_* = \{ \sup_{\bm{y}} \bm{z}^\top \bm{y} \mid \| \bm{y} \| \leq 1 \},
\end{equation*}
and the fact that  $\ell_q$ norm is dual of $\ell_p$ norm for $p, q \in [1,\infty]$, the term $\left( \max_{\bm{y} \in B_p(\bm{0},1)} \A{(0)}_{j,:}\bm{y} \right)$ in \eqref{proof:cor:max2} can be expressed as $\| \A{(0)}_{j,:}\|_q$ in \eqref{proof:cor:max3}. Similarly, 
\begin{align}
\gamma^L_j := 
\min_{\x \in B_p(\xo,\epsilon)} f_j^{L}(\x) 
&= \min_{\x \in B_p(\xo,\epsilon)} \left( \A{(0)}_{j,:}\x+ \bias{(m)}_j+\sum_{k=1}^{m-1}\A{(k)}_{j,:}(\bias{(k)}-\lwbias{(k)}_{:,j}) \right) \nonumber \\
&= \left( \min_{\x \in B_p(\xo,\epsilon)} \A{(0)}_{j,:}\x \right) + \bias{(m)}_j+\sum_{k=1}^{m-1}\A{(k)}_{j,:}(\bias{(k)}-\lwbias{(k)}_{:,j}) \nonumber\\
&= \epsilon \left( \min_{\bm{y} \in B_p(\bm{0},1)} \A{(0)}_{j,:}\bm{y} \right) + \A{(0)}_{j,:}\xo + \bias{(m)}_j+\sum_{k=1}^{m-1}\A{(k)}_{j,:}(\bias{(k)}-\lwbias{(k)}_{:,j}) \nonumber\\
&= -\epsilon \left( \max_{\bm{y} \in B_p(\bm{0},1)} -\A{(0)}_{j,:}\bm{y} \right) + \A{(0)}_{j,:}\xo + \bias{(m)}_j+\sum_{k=1}^{m-1}\A{(k)}_{j,:}(\bias{(k)}-\lwbias{(k)}_{:,j}) \label{proof:cor:min1}\\
&= -\epsilon \| \A{(0)}_{j,:}\|_q + \A{(0)}_{j,:}\xo + \bias{(m)}_j+\sum_{k=1}^{m-1}\A{(k)}_{j,:}(\bias{(k)}-\lwbias{(k)}_{:,j}). \label{proof:cor:min2}
\end{align}
Again, from \eqref{proof:cor:min1} to \eqref{proof:cor:min2}, we simply replace $\left( \max_{\bm{y} \in B_p(\bm{0},1)} -\A{(0)}_{j,:}\bm{y} \right)$ by $\| -\A{(0)}_{j,:}\|_q = \| \A{(0)}_{j,:}\|_q$. 
Thus, if we use $\nu_j$ to denote the common term $\A{(0)}_{j,:} \xo +  \bias{(m)}_j + \sum_{k=1}^{m-1}\A{(k)}_{j,:}\bias{(k)}$, we have 
\begin{align*}
	\gamma^U_j = \epsilon \|\A{(0)}_{j,:}\|_q - \sum_{k=1}^{m-1}\A{(k)}_{j,:}\upbias{(k)}_{:,j} + \nu_j, \quad \quad & \text{(upper bound)}\\
	\gamma^L_j = -\epsilon \|\A{(0)}_{j,:}\|_q - \sum_{k=1}^{m-1}\A{(k)}_{j,:}\lwbias{(k)}_{:,j} + \nu_j. \; \quad & \text{(lower bound)}
\end{align*}

\newpage
\section{Algorithms}
\label{sec:app_algs}
We present our full algorithms, \fastlin in Algorithm~\ref{alg:fast-lin} and \fastlip in Algorithm~\ref{alg:fast-lip}.

\newpage

\section{An alternative bound on the Lipschitz constant}\label{app:alternative_bound_on_Lipschitz}
Using the property of norm, we can derive an upper bound of the gradient norm of a $2$-layer ReLU network in the following: 
\begin{align}
	& \quad \| \nabla f_j(\x) \|_q \nonumber \\
    &= \| \W{(2)}_{j,:} \Lam{(1)} \W{(1)} \|_q \nonumber \\
    &= \| \W{(2)}_{j,:} (\Lam{(1)}_a+\Lam{(1)}_u) \W{(1)} \|_q  \label{eq:grad_qbnd1} \\
    & \leq \| \W{(2)}_{j,:} \Lam{(1)}_a \W{(1)} \|_q + \| \W{(2)}_{j,:} \Lam{(1)}_u \W{(1)} \|_q  \label{eq:grad_qbnd2} \\
    & \leq \| \W{(2)}_{j,:} \Lam{(1)}_a \W{(1)} \|_q +  \sum_{r \in \setIuns{1}} \| \W{(2)}_{j,r} \W{(1)}_{r,:} \|_q \label{eq:grad_qbnd3}
\end{align}
where with a slight abuse of notation, we use $\Lam{(1)}_a$ to denote the diagonal activation matrix for neurons who are always activated, i.e. its $(r,r)$ entry $\Lam{(1)}_{a(r,r)}$ is $1$ if 
$r \in \setIpos{1}$ and $0$ otherwise, and we use $\Lam{(1)}_u$ to denote the diagonal activation matrix for neurons whose status are uncertain because they could possibly be active or inactive, i.e. its $(r,r)$ entry $\Lam{(1)}_{u(r,r)}$ is $1$ if $r \in \setIuns{1}$ and $0$ otherwise. Therefore, we can write $\Lam{(1)}$ as a sum of $\Lam{(1)}_a$ and $\Lam{(1)}_u$. 

Note that \eqref{eq:grad_qbnd1} to \eqref{eq:grad_qbnd2} is from the sub-additive property of a norm, and \eqref{eq:grad_qbnd2} to \eqref{eq:grad_qbnd3} uses the sub-additive property of a norm again and set the uncertain neurons encoding all to $1$ because $$\| \W{(2)}_{j,:} \Lam{(1)}_u \W{(1)} \| = \| \sum_{r \in \setIuns{1}} \W{(2)}_{j,r} \Lam{(1)}_{u(r,r)} \W{(1)}_{r,:} \| \leq \sum_{r \in \setIuns{1}}  \| \W{(2)}_{j,r} \Lam{(1)}_{u(r,r)} \W{(1)}_{r,:} \| \leq \sum_{r \in \setIuns{1}} \| \W{(2)}_{j,r}  \W{(1)}_{r,:} \|.$$
Notice that \eqref{eq:grad_qbnd3} can be used as an upper bound of Lipschitz constant and is applicable to compute a certified lower bound for minimum adversarial distortion of a general $\ell_p$ norm attack. However, this bound is expected to be less tight because we simply include all the uncertain neurons to get an upper bound on the norm in \eqref{eq:grad_qbnd3}.


\newpage

\section{Details of Experiments in Section~\ref{sec:exp}}
\label{app:exp}

\subsection{Methods}
Below, we give detailed descriptions on the methods that we compare in Table~\ref{tb:smallnetwork_and_large}, Table~\ref{tb:smallnetwork} and Table~\ref{tb:largenetwork_app}: 
\begin{itemize}
\item \fastlin: Our proposed method of directly bounding network output via \textbf{lin}ear upper/lower bounds for ReLU, as discussed in Section~\ref{sec3:convexbnd} and Algorithm~\ref{alg:fast-lin};
\item \fastlip: Our proposed method based on bounding local \textbf{Lip}schitz constant, in Section \ref{sec3:gradbnd} and Algorithm~\ref{alg:fast-lip};
\item \reluplex: \reluplex~\cite{katz2017reluplex} is a satisfiability modulo theory (SMT) based solver which delivers a true minimum distortion, but is very computationally expensive; 
\item \lpfull: A linear programming baseline method with formulation borrowed from~\cite{zico17convex}. Note that we solve the primal LP formulation \textit{exactly} to get a best possible bound. This variant solves \textbf{full} relaxed \lp problems at \textit{every} layer to give a final ``adversarial polytope''. Similar to our proposed methods, it only gives a lower bound. We extend this formulation to $p=2$ case, where the input constraint becomes quadratic and requires a quadratic constrained programming (QCP) solver, which is usually slower than LP solvers. 
\item \lp: Similar to \lpfull, but this variant solves only \textbf{one} LP problem for the full network at the output neurons and the layer-wise bounds for the neurons in hidden layers are solved by \fastlin. We also extend it to $p=2$ case with QCP constraints on the inputs. \lp and \lpfull are served as our baselines to compare with \fastlin and \fastlip;
\item \textbf{Attacks}: Any successful adversarial example gives a valid \textit{upper bound} for the minimum adversarial distortion. For larger networks where \reluplex is not feasible, we run adversarial attacks and obtain an upper bound of minimal adversarial distortions to compare with. We apply the $\ell_2$ and $\ell_\infty$ variants of Carlini and Wagner's attack (CW)~\cite{carlini2017towards} to find the best $\ell_2$ and $\ell_\infty$ distortions. We found that the CW $\ell_\infty$ attack usually finds adversarial examples with smaller $\ell_\infty$ distortions than using PGD (projected gradient descent). We use EAD~\cite{chen2017ead}, a Elastic-Net regularized attack, to find adversarial examples with small $\ell_1$ distortions. We run CW $\ell_2$ and $\ell_\infty$ attacks for 3,000 iterations and EAD attacks for 2,000 iterations; 
\item \clever: \clever~\cite{weng2017evaluating} is an attack-agnostic robustness score based on local Lipschitz constant estimation and provides an estimated lower-bound. It is capable of performing robustness evaluation for large-scale networks but is not a certified lower bound;
\item \opnorm: Operator norms of weight matrices were first used in~\cite{szegedy2013intriguing} to give a robustness lower bound. We compute the $\ell_p$ induced norm of weight matrices of each layer and use their product as the global Lipschitz constant $L_q^j$. A valid lower bound is given by $g(\bm{x_0})/L_q^j$ (see Section \ref{sec3:gradbnd}). We only need to pre-compute the operator norms once for all the examples.
\end{itemize}

\subsection{Setup}
We use MNIST and CIFAR datasets and evaluate the performance of each method in   MLP networks with up to 7 layers or over 10,000 neurons, which is the largest network size for non-trivial and guaranteed robustness verification to date. We use the same number of hidden neurons for each layer and denote a $m$-layer network with $n$ hidden neurons in each layer as $m\times[n]$. Each network is trained with a grid search of learning rates from $\{0.1, 0.05, 0.02, 0.01, 0.005\}$ and weight decays from $\{10^{-4}, 10^{-5}, 10^{-6}, 10^{-7}, 10^{-8}\}$ and we select the network with the best validation accuracy. We consider both targeted and untargeted robustness under $\ell_p$ distortions ($p=1,2,\infty$); for targeted robustness, we consider three target classes: a random class, a least likely class and a runner-up class (the class with second largest probability). The reported average scores are an average of 100 images from the test set, with images classified wrongly skipped. Reported time is per image. We use binary search to find the certified lower bounds in \fastlin, \fastlip, \lp and \lpfull, and the maximum number of search iterations is set to 15.

We implement our algorithm using Python (with Numpy and Numba)\footnote{\url{https://github.com/huanzhang12/CertifiedReLURobustness}}, while for the LP based method we use the highly efficient Gurobi commercial LP solver with Python Interface. All experiments are conducted in single thread mode (we disable the concurrent solver in Gurobi) on a Intel Xeon E5-2683v3 (2.0 GHz) CPU. Despite the inefficiency of Python, we still achieve two orders of magnitudes speedup compared with \lp, while achieving a very similar lower bound. Our methods are automatically parallelized by Numba and can gain further speedups on a multi-core CPU, but we disabled this parallelization for a fair comparison to other methods.

\subsection{Discussions}
In Table~\ref{tb:smallnetwork_and_large}a (full Table: Table~\ref{tb:smallnetwork}), we compare the lower bound $\beta_L$ computed by each algorithm to the true minimum distortion $r_0$ found by \reluplex. We are only able to verify 2 and 3 layer MNIST with 20 neurons per hidden layer within reasonable time using \reluplex. It is worth noting that the input dimension (784) is very large compared to the network evaluated in~\cite{katz2017reluplex} with only 5 inputs. Lower bounds found by \fastlin is very close to \lp, and the gaps are within 2-3X from the true minimum distortion $r_0$ found by \reluplex. The upper bound given by CW $\ell_\infty$ are also very close to $r_0$.

In Table~\ref{tb:smallnetwork_and_large}b (full Table: Table~\ref{tb:largenetwork_app}), we compare \fastlin, \fastlip with \lp and \opnorm on larger networks with up to over ten thousands hidden neurons. \fastlin and \fastlip are significantly faster than \lp and are able to verify much larger networks (\lp becomes very slow to solve exactly on 4-layer MNIST with 4096 hidden neurons, and is infeasible for even larger CIFAR models). \fastlin achieves a very similar bound comparing with results of \lp over all smaller models, but being \textit{over two orders of magnitude faster}. We found that \fastlip can achieve better bounds when $p = 1$ in two-layers networks, and is comparable to \fastlin in shallow networks. Meanwhile, we also found that \fastlin scales better than \fastlip for deeper networks, where \fastlin usually provides a good bound even when the number of layers is large. For deeper networks, neurons in the last few layers are likely to have uncertain activations, making \fastlip being too pessimistic. However, \fastlip outperforms the global Lipschitz constant based bound (\opnorm) which quickly goes down to 0 when the network goes deeper, as \fastlip is bounding the \textit{local} Lipschitz constant to compute robustness lower bound. In Table~\ref{tb:largenetwork_app}, we also apply our method to MNIST and CIFAR models to compare the minimum distortion for \textit{untargeted} attacks. The computational benefit of \fastlin and \fastlip is more significant than $\lp$ because \lp needs to solve $n_m$ objectives (where $n_m$ is the total number of classes), whereas the cost of our methods stay mostly unchanged as we get the bounds for all network outputs simultaneously.

In Table~\ref{tb:distill}, we compute our two proposed lower bounds on neural networks with defending techniques to evaluate the effects of defending techniques (e.g. how much robustness is increased). We train the network with two defending methods, defensive distillation (DD)~\cite{papernot2016distillation} and adversarial training~\cite{madry2017towards} based on robust optimization. For DD we use a temperature of 100, and for adversarial training, we train the network for 100 epochs with adversarial examples crafted by 10 iterations of PGD with $\epsilon=0.3$. The test accuracy for the adversarially trained models dropped from 98.5\% to 97.3\%, and from 98.6\% to 98.1\%, for 3 and 4 layer MLP models, respectively. We observe that both defending techniques can increase the computed robustness lower bounds, however adversarial training is significantly more effective than defensive distillation. The lower bounds computed by \fastlin are close to the desired robustness guarantee $\epsilon=0.3$.


\begin{table*}[htbp]
\centering
\caption{Comparison of our proposed certified lower bounds \fastlin and \fastlip, \lp and \lpfull, the estimated lower bounds by \clever, the exact minimum distortion by \reluplex, and the upper bounds by \textbf{Attack} algorithms (CW $\ell_\infty$ for $p = \infty$, CW $\ell_2$ for $p = 2$, and EAD for $p = 1$) on 2, 3 layers toy MNIST networks with \textit{only 20 neurons per layer}. Differences of lower bounds and speedup are measured on the two corresponding \textbf{bold} numbers in each row, representing the best answer from our proposed algorithms and LP based approaches. \reluplex is designed to verify $\ell_\infty$ robustness so we omit results for $\ell_2$ and $\ell_1$. Note that \lpfull and \reluplex \textit{are very slow} and cannot scale to any practical networks, and the purpose of this table is to show how close our fast bounds are compared to the true minimum distortion provided by \textbf{Reluplex} and the bounds that are slightly tighter but very expensive (e.g. \lpfull).}
\label{tb:smallnetwork}

\begin{subtable}[t]{\textwidth}

\raggedright
\scalebox{0.76}{
\begin{tabular}{|c|c|c|
>{\centering\arraybackslash}p{9ex}
>{\centering\arraybackslash}p{9ex}|
>{\centering\arraybackslash}p{9ex}
>{\centering\arraybackslash}p{9ex}
||c||c||cc|}
\hline
\multicolumn{3}{|c|}{Toy Networks}                                                                                                   & \multicolumn{8}{c|}{Average Magnitude of Distortions on 100 Images}                                                                        \\ \hline
\multirow{3}{*}{Network}                                                       & \multirow{3}{*}{$p$}   & \multirow{3}{*}{Target}  & \multicolumn{4}{c||}{Certified Bounds}                                      & difference       & Exact                   & \multicolumn{2}{c|}{Uncertified}                    \\ \cline{4-7}\cline{9-11}
                                                                               &                             &                          & \multicolumn{2}{c|}{Our bounds} & \multicolumn{2}{c||}{Our baselines} & ours vs.        & Reluplex                & CLEVER                    &  Attacks                \\ \cline{4-7}
                                                                               &                             &                          & Fast-Lin     & Fast-Lip         & LP         & LP-Full                      & LP(-Full)         & \cite{katz2017reluplex} & \cite{weng2017evaluating} & CW/EAD                  \\ \hline

\multirow{9}{*}{\begin{tabular}[c]{@{}c@{}}MNIST\\ $2\times[20]$\end{tabular}} & \multirow{3}{*}{$\infty$} & runner-up & \bf 0.0191 & 0.0167     & \bf 0.0197 & 0.0197     &  -3.0\% & 0.04145  & 0.0235 & 0.04384 \\
                                                                               &                           & rand      & \bf 0.0309 & 0.0270     & \bf 0.0319 & 0.0319     &  -3.2\% & 0.07765  & 0.0428 & 0.08060 \\
                                                                               &                           & least     & \bf 0.0448 & 0.0398     & \bf 0.0462 & 0.0462     &  -3.1\% & 0.11711  & 0.0662 & 0.1224 \\ \cline{2-11}
                                                                               & \multirow{3}{*}{$2$}      & runner-up & \bf 0.3879 & 0.3677     & 0.4811     & \bf 0.5637 & -31.2\% & -        & 0.4615 & 0.64669 \\
                                                                               &                           & rand      & \bf 0.6278 & 0.6057     & 0.7560     & \bf 0.9182 & -31.6\% & -        & 0.8426 & 1.19630 \\
                                                                               &                           & least     & \bf 0.9105 & 0.8946     & 1.0997     & \bf 1.3421 & -32.2\% & -        & 1.315  & 1.88830 \\ \cline{2-11}
                                                                               & \multirow{3}{*}{$1$}      & runner-up & 2.3798     & \bf 2.8086 & 2.5932     & \bf 2.8171 &  -0.3\% & -        & 3.168  & 5.38380 \\
                                                                               &                           & rand      & 3.9297     & \bf 4.8561 & 4.2681     & \bf 4.6822 &  +3.7\% & -        & 5.858  & 11.4760 \\
                                                                               &                           & least     & 5.7298     & \bf 7.3879 & 6.2062     & \bf 6.8358 &  +8.1\% & -        & 9.250  & 19.5960 \\ \hline
\multirow{9}{*}{\begin{tabular}[c]{@{}c@{}}MNIST\\ $3\times[20]$\end{tabular}} & \multirow{3}{*}{$\infty$} & runner-up & \bf 0.0158 & 0.0094     & 0.0168     & \bf 0.0171 &  -7.2\% & 0.04234  & 0.0223 & 0.04786 \\
                                                                               &                           & rand      & \bf 0.0229 & 0.0142     & 0.0241     & \bf 0.0246 &  -6.9\% & 0.06824  & 0.0385 & 0.08114 \\
                                                                               &                           & least     & \bf 0.0304 & 0.0196     & 0.0319     & \bf 0.0326 &  -6.9\% & 0.10449  & 0.0566 & 0.11213 \\ \cline{2-11}
                                                                               & \multirow{3}{*}{$2$}      & runner-up & \bf 0.3228 & 0.2142     & 0.3809     & \bf 0.4901 & -34.1\% & -        & 0.4231 & 0.74117 \\
                                                                               &                           & rand      & \bf 0.4652 & 0.3273     & 0.5345     & \bf 0.7096 & -34.4\% & -        & 0.7331 & 1.22570 \\
                                                                               &                           & least     & \bf 0.6179 & 0.4454     & 0.7083     & \bf 0.9424 & -34.4\% & -        & 1.100  & 1.71090 \\ \cline{2-11}
                                                                               & \multirow{3}{*}{$1$}      & runner-up & \bf 2.0189 & 1.8819     & 2.2127     & \bf 2.5010 & -19.3\% & -        & 2.950  & 6.13750 \\
                                                                               &                           & rand      & \bf 2.8550 & 2.8144     & 3.1000     & \bf 3.5740 & -20.1\% & -        & 4.990  & 10.7220 \\
                                                                               &                           & least     & 3.7504     & \bf 3.8043 & 4.0434     & \bf 4.6967 & -19.0\% & -        & 7.131  & 15.6850 \\ \hline
\end{tabular}
}
\caption{Comparison of bounds}
\end{subtable}
\begin{subtable}[t]{\textwidth}
\centering
\raggedright
\scalebox{0.76}{
\begin{tabular}{|c|c|c|
>{\centering\arraybackslash}p{9ex}
>{\centering\arraybackslash}p{9ex}|
>{\centering\arraybackslash}p{9ex}
>{\centering\arraybackslash}p{9ex}
||c||c|}
\hline
\multicolumn{3}{|c|}{Toy Networks}                                                                                                     & \multicolumn{6}{c|}{Average Running Time per Image}                                                             \\ \hline
\multirow{3}{*}{Network}                                                       & \multirow{2}{*}{$p$}        & \multirow{3}{*}{Target} & \multicolumn{4}{c||}{Certified Bounds}                                      &       Exact     & Speedup         \\ \cline{4-8}
                                                                               &                             &                         & \multicolumn{2}{c|}{Our bounds} & \multicolumn{2}{c||}{Our baselines} &    Reluplex     & ours vs.       \\ \cline{4-7}

                                                                               &                             &           & Fast-Lin    & Fast-Lip    & LP         & LP-Full    & \cite{katz2017reluplex} &    LP-(full)     \\ \hline
\multirow{9}{*}{\begin{tabular}[c]{@{}c@{}}MNIST\\ $2\times[20]$\end{tabular}} & \multirow{3}{*}{$\infty$}   & runner-up & \bf 3.09 ms & 3.49 ms     & \bf 217 ms & 1.74 s     & 134 s    & 70X    \\
                                                                               &                             & rand      & \bf 3.25 ms & 5.53 ms     & \bf 234 ms & 1.93 s     & 38 s     & 72X    \\
                                                                               &                             & least     & \bf 3.37 ms & 8.90 ms     & \bf 250 ms & 1.97 s     & 360 s    & 74X    \\ \cline{2-9}
                                                                               & \multirow{3}{*}{$2$}        & runner-up & \bf 3.00 ms & 3.76 ms     & 1.10 s     & \bf 20.6 s & -        & 6864X  \\
                                                                               &                             & rand      & \bf 3.37 ms & 6.16 ms     & 1.20 s     & \bf 23.1 s & -        & 6838X  \\
                                                                               &                             & least     & \bf 3.29 ms & 9.89 ms     & 1.27 s     & \bf 26.4 s & -        & 8021X  \\ \cline{2-9}
                                                                               & \multirow{3}{*}{$1$}        & runner-up & 2.85 ms     & \bf 39.2 ms & 1.27 s     & \bf 16.1 s & -        & 412X   \\
                                                                               &                             & rand      & 3.32 ms     & \bf 54.8 ms & 1.59 s     & \bf 17.3 s & -        & 316X   \\
                                                                               &                             & least     & 3.46 ms     & \bf 68.1 ms & 1.74 s     & \bf 17.7 s & -        & 260X   \\ \hline
\multirow{9}{*}{\begin{tabular}[c]{@{}c@{}}MNIST\\ $3\times[20]$\end{tabular}} & \multirow{3}{*}{$\infty$}   & runner-up & \bf 5.58 ms & 3.64 ms     & 253 ms     & \bf 6.12 s & 4.7 hrs  & 1096X  \\
                                                                               &                             & rand      & \bf 6.12 ms & 5.23 ms     & 291 ms     & \bf 7.16 s & 11.6 hrs & 1171X  \\
                                                                               &                             & least     & \bf 6.62 ms & 7.06 ms     & 307 ms     & \bf 7.30 s & 12.6 hrs & 1102X  \\ \cline{2-9}
                                                                               & \multirow{3}{*}{$2$}        & runner-up & \bf 5.35 ms & 3.95 ms     & 1.22 s     & \bf 57.5 s & -        & 10742X \\
                                                                               &                             & rand      & \bf 5.86 ms & 5.81 ms     & 1.27 s     & \bf 66.3 s & -        & 11325X \\
                                                                               &                             & least     & \bf 5.94 ms & 7.55 ms     & 1.34 s     & \bf 77.3 s & -        & 13016X \\ \cline{2-9}
                                                                               & \multirow{3}{*}{$1$}        & runner-up & \bf 5.45 ms & 39.6 ms     & 1.27 s     & \bf 75.0 s & -        & 13763X \\
                                                                               &                             & rand      & \bf 5.56 ms & 52.9 ms     & 1.47 s     & \bf 82.0 s & -        & 14742X \\
                                                                               &                             & least     & 6.07 ms     & \bf 65.9 ms & 1.68 s     & \bf 85.9 s & -        & 1304X  \\ \hline
\end{tabular}
}
\caption{Comparison of time}
\end{subtable}
\end{table*}

\end{document}